\documentclass[twoside]{article}

\usepackage[accepted]{aistats2021}
\usepackage[round]{natbib}

\usepackage[utf8]{inputenc} 
\usepackage[T1]{fontenc}    
\usepackage{hyperref}       
\usepackage{url}            
\usepackage{booktabs}       
\usepackage{amsfonts}       
\usepackage{nicefrac}       
\usepackage{microtype}      
\usepackage{graphicx}
\usepackage{subfigure}
\usepackage{booktabs} 
\usepackage{amsmath,amsfonts,amssymb,amsthm}
\usepackage{dsfont}
\usepackage{xcolor}
\usepackage{mathrsfs}
\usepackage{centernot}
\usepackage{mathtools}
\usepackage{mathrsfs}
\usepackage{breakcites}
\usepackage{wrapfig}

\usepackage[ruled,vlined]{algorithm2e}

\newtheorem{theorem}{Theorem}
\newtheorem{prop}{Proposition}
\newtheorem{lemma}{Lemma}

\begin{document}

%

%

\twocolumn[
\aistatstitle{Deep Fourier Kernel for Self-Attentive Point Processes}
\aistatsauthor{Shixiang Zhu \And Minghe Zhang \And Ruyi Ding \And Yao Xie}
\aistatsaddress{Georgia Institute of Technology} ]

\begin{abstract}
We present a novel attention-based model for discrete event data to capture complex non-linear temporal dependence structures. We borrow the idea from the attention mechanism and incorporate it into the point processes' conditional intensity function. We further introduce a novel score function using Fourier kernel embedding, whose spectrum is represented using neural networks, which drastically differs from the traditional dot-product kernel and can capture a more complex similarity structure. We establish our approach's theoretical properties and demonstrate our approach's competitive performance compared to the state-of-the-art for synthetic and real data.
\end{abstract}

\vspace{-0.1in}
\section{Introduction}
\label{sec:introduction}
\vspace{-0.1in}

Discrete event data are ubiquitous in modern applications, ranging from traffic incidents, police incidents, user behaviors in social networks, and earthquake catalogs. Such data consist of a sequence of events that indicate when and where each event occurred and any additional descriptive information about the event (such as category, marks, and free-text). The distribution of events is of scientific and practical interest, both for prediction purposes and for inferring these events' underlying generative mechanism. 

A popular framework for modeling discrete events is point processes, which can be continuous over time and space. Multi-dimensional point processes can be used to model discrete events over networks. An important aspect of this model is to capture the triggering or inhibiting effect of the event on subsequent events in the future.  Since the distribution of point processes is completely specified by the conditional intensity function (the occurrence rate of events conditioning on their history), such triggering effect can be conveniently modeled by assuming parametric forms. In the classical statistical framework, the conditional intensity function usually consists of a deterministic background rate plus a stochastic term that includes the influence of the historical events, characterized by a triggering kernel function. For example, the seminar work \citep{Ogata1998} proposed the epidemic-type aftershock sequence (ETAS), which suggests an exponentially decaying function over the temporal and spatial distance between events. However, with the increasing complexity and quantity of modern data, we need more expressive models. Recently, there has been much effort in developing neural network-based point processes, leveraging the rich representation power of neural networks \citep{Omi2019, Zhu2019B}. In particular, because of the sequential nature of event data, existing methods rely heavily on Recurrent Neural Networks (RNNs) \citep{Du2016, Li2018, Mei2017, Upadhyay2018, Xiao2017B, Xiao2017A, Zhu2020}.

However, there is a notable limitation of existing neural network-based models. 
The popular RNN models such as Long Short-Term Memory (LSTM) \citep{Hochreiter1997} are not enough capable of capturing long-range dependencies and still implicitly assumes that the influence of the current event decays monotonically over time (due to their recursive structure).   
Many real-world applications may not be good candidates to apply these assumptions. For instance, in modeling economic time-series, major economic or historical events (such as economic crisis or shift of policy) will have a much longer impact; their influence may be carried over to current time and should not be ``forgotten'' by the event model. In modeling traffic events, when a major car accident occurs on the highway, it takes hours to clear the scene. The congestion will not ease during that period -- the influence of major traffic incident events may not decay monotonically over time. These motivate us to tackle the long-term and non-homogeneous influence function and capture the influence of past events in a more flexible manner. 

\begin{figure}[!t]
\begin{center}
  \subfigure[RNN-based Point Process]{\includegraphics[width=.75\linewidth]{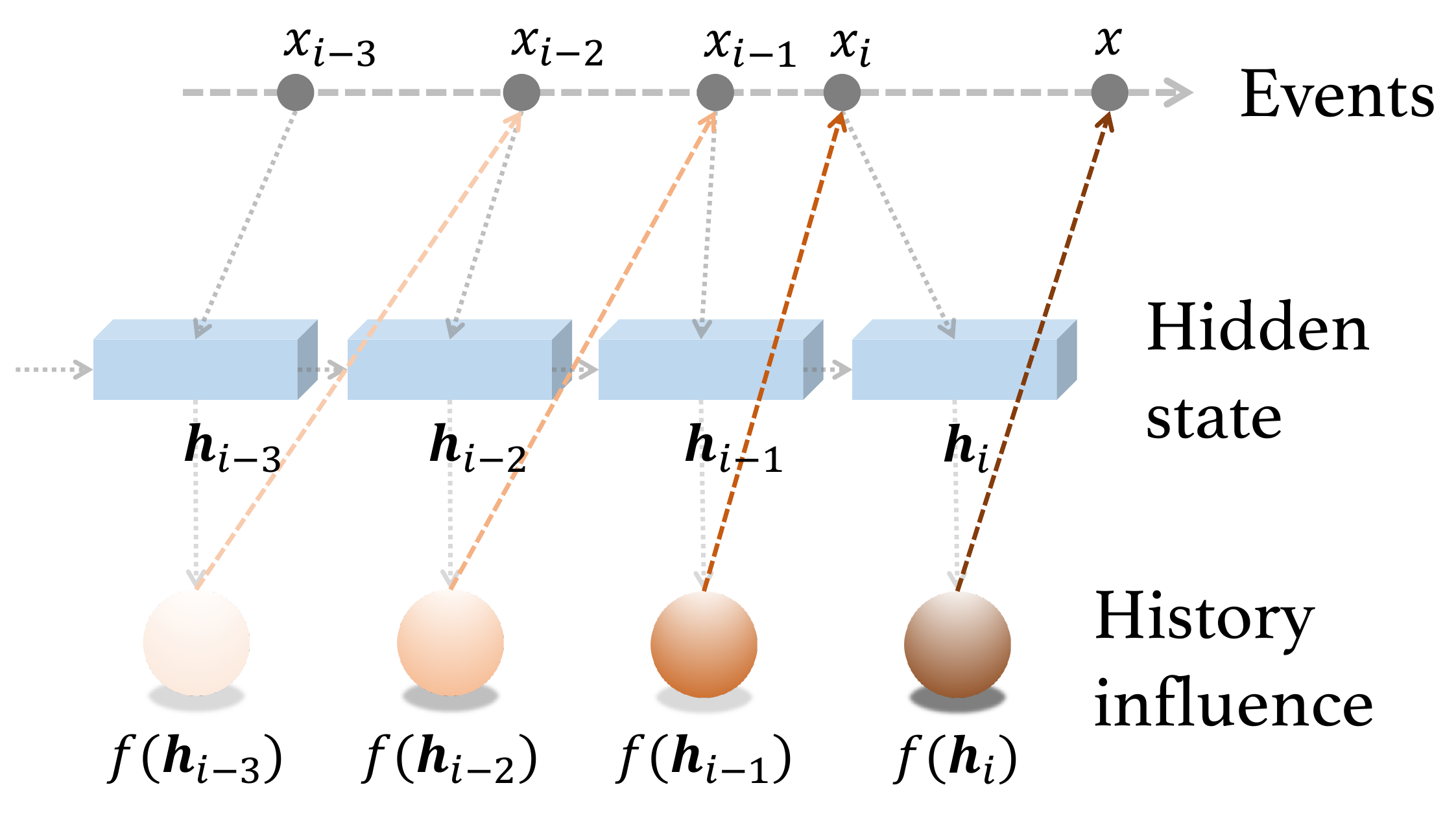}}
\vfill
  \subfigure[Deep Attention Point Process]{\includegraphics[width=.75\linewidth]{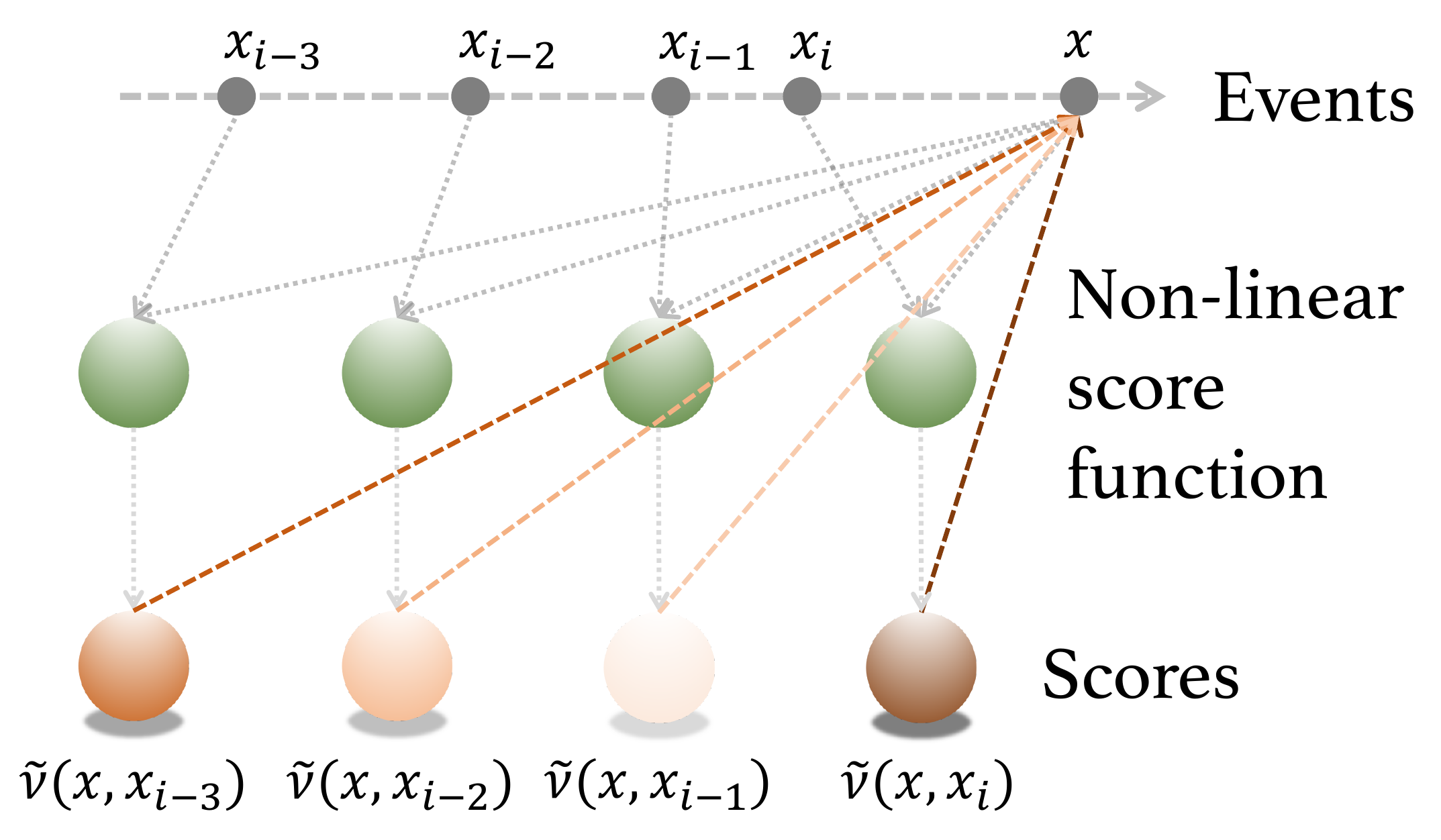}}
  \caption{\small{Comparison between RNN-based models and our DAPP. The color depth of the red balls represent their ``importance'' in the model. The history influence in (a) are exponentially decaying over the time. The score is a non-linear function with respect to the distance between events and is non-homogeneous over the time.}}
  \label{fig:rnn-attention-comparison}
\end{center}
\vspace{-0.1in}
\end{figure}

In the domain of natural language processing (NLP) and computer vision, the self-attention mechanism has been widely adopted as an algorithmic component to tackle the effect of non-linear and long-range dependence \citep{Vaswani2017}. This motivates us to adapt the attention mechanism for the point processes models, leveraging their capabilities to capture long-range and complex dependency in the sequence. However, since the attention mechanism has rarely been used outside of the domains mentioned above, we still need to develop a principled probabilistic (stochastic process) model framework to incorporate the attention mechanism into continuous point processes properly. 
In particular, unlike the NLP problem \citep{Mikolov2013}, where the similarity between words can be adequately characterized by \emph{dot-product} score \citep{Vaswani2017} in the conventional attention mechanism, discrete events usually exhibit heterogeneous triggering effects regarding their spatio-temporal distances. 
Take earthquake catalog data as an example. 
The dynamics between seismic events are related to the geologic structure of faults.
For instance, most aftershocks either occur along the fault plane or other faults within the volume affected by the strain associated with the mainshock.

In this paper, we propose a deep attention point process (DAPP) model, with a flexible non-linear score function based on Fourier kernels in the attention mechanism, as shown in Figure~\ref{fig:rnn-attention-comparison}. We go beyond the recurrent structure of RNN that the historical information can only be passed through the hidden state. Instead, we leverage the attention mechanism to develop a flexible framework that ``focuses'' on past events with high ``importance'' scores, regardless of how far away they are. We also present a novel score function via Fourier kernels with spectrum represented using deep neural networks, whose parameters are learned from data. In contrast to the commonly used dot-product score, which essentially performs linear key embedding, our score function performs non-linear kernel-induced feature embedding, capturing more complex similarity structures in events. This can help achieve higher flexibility in retaining the most ``significant'' historical events relative to the current event. Moreover, to achieve constant memory in the face of streaming data, we develop an online version of DAPP, which is more suitable to process streaming data. We establish the Fourier kernel's theoretical properties and demonstrate the competitive performance of our proposed method relative to the state-of-the-art on a wide range of real and synthetic data sets. 

Our contributions include (1) introducing a general probabilistic attention-based point process model for discrete event data; (2) introducing a novel similarity kernel based on Fourier kernel embedding and neural-network represented spectrum (in contrast to the standard dot-product kernel). 

\vspace{-0.1in}
\paragraph{Related work.}
Existing works for point processes modeling, such as \cite{Gomez2010, Yuan2019, Zhu2019A}, often assuming parametric forms of the intensity functions. Such methods enjoy good interpretability and are efficient to perform. However, parametric models are not expressive enough to capture the events' dynamics in some applications. 
As previously mentioned, recent interest has focused on improving the expressive power of point process models using RNNs \citep{Du2016, Li2018, Mei2017, Upadhyay2018, Zhu2020}. 
However, the events' dependence in the conditional intensity is specified as a parametric form. For instance, \cite{Du2016} expresses the influence of two consecutive events in a form of $\exp \{ w (t_{i+1} - t_{i})\}$, which is an exponential function with respect to the length of the time interval.


There also have been some works that model stochastic processes using the attention mechanism \citep{Zhang2019, Kim2019}. \cite{Kim2019} uses self-attention to model a class of neural latent variable models, called Neural Processes \citep{Garnelo2018}, which is not for sequential data specifically.
In retrospect, we realize a concurrent work \citep{Zhang2019, zuo2020transformer} which also use the attentive mechanism to model point processes, but the framework is different. An important distinction of their approach from ours is that they rely on a dot-product between features (embedded in a Gaussian argument) in the attention mechanism; we use a more flexible and general Fourier kernel for similarity function, and we also specifically address the design and learning of the kernel by representing the spectrum of the Fourier kernel using neural networks. 

\vspace{-0.1in}
\section{Background} 
\label{sec:point-process}
\vspace{-0.1in}



Marked temporal point processes (MTPPs) \citep{Reinhart2017} consist of an ordered sequence of events localized in time, location, and mark spaces. Let $\{x_1 , x_2 , \dots, x_{N_T}\}$ represent a sequence of points sampled from a MTPP. We denote $N_T$ as the number of the points generated in the time horizon $[0, T)$. Each point $x_i$ is a marked spatio-temporal tuple $x_i = (t_i, m_i)$, where $t_i \in [0, T)$ is the time of occurrence of the $i$-th event, and $m_i \in \mathcal{M}$ is the corresponding mark, which may contain location, event type, or other rich description information (such as image or free-text). Here we treat discrete location marks, while sometimes the continuous location is treated separately in spatio-temporal point processes. 

The events' distribution in MTPPs are characterized via a conditional intensity function $\lambda(t, m|\mathcal{H}_t)$, which is the probability of observing an event in the marked temporal space $[0, T) \times \mathcal{M}$ given the events' history $\mathcal{H}_t = \{ (t_i, m_i)|t_i < t \}$, i.e.,
$
    \lambda(t, m | \mathcal{H}_t) |B(m, dm)| dt
    = \mathbb{E}\left[ N([t, t+dt) \times B(m, dm)) | \mathcal{H}_t \right]
$,
where $N(A)$ is the counting measure of events over the set $A \subseteq \mathcal{X}$ and $|B(m, dm)|$ is the Lebesgue measure of the ball $B(m, dm)$ with radius $dm$.
The log-likelihood of observing a sequence with $n$ events denoted as $\boldsymbol{x} = \{(t_i, m_i)\}_{i=1}^{N_T}$ can be obtained by 
\begin{equation}
    \ell(\boldsymbol{x}) =  \sum_{i=1}^{N_T} \log \lambda(t_i, m_i | \mathcal{H}_{t_{i}}) - \int_{m \in \mathcal{M}} \int_0^T \lambda(t, m | \mathcal{H}_{t}) dt dm.
    \label{eq:log-likelihood}
\end{equation}

As self- and mutual-exciting point processes, Hawkes processes \citep{Hawkes1971} have been widely used to capture the mutual excitation dynamics among temporal events. The model assumes that influences from past events are linearly additive towards the current event. The conditional intensity function of a Hawkes process is defined as 
\begin{equation}
    \lambda(t, m | \mathcal{H}_{t}) = \mu + \sum_{t_i < t} g(t - t_i, m - m_i),
    \label{eq:hawkes}
\end{equation}
where $\mu \ge 0$ is the background intensity of events, $g(\cdot) \ge 0$ is the triggering function that captures spatio-temporal and marked dependencies of the past events. The triggering function can be chosen in advance, e.g., in one-dimensional cases, $g(t,t_i) = \alpha \exp \{ - \beta (t - t_i) \}$, where $\beta$ controls the decay rate and $\alpha >0$ controls the magnitude of the influence. 

\section{Proposed Method}
\label{sec:methods}
\vspace{-0.1in}

This section presents a novel attention-based point process model using the deep Fourier kernel as its score function, which is capable of remembering long-term memory and capturing non-homogeneous triggering effects.

\vspace{-0.1in}
\subsection{Self-attention in point processes}
\label{sec:attention-pp}
\vspace{-0.1in}


Deep Attention Point Processes (DAPP) aims to model the current event's nonlinear dependencies from past events using the attention mechanism. 
Specifically, we model the conditional intensity function of MTPPs using the attention output. 
DAPP also adopts the `` multi-heads'' mechanism, which offers multiple ``representation subspaces'' for events in the sequence. We describe the DAPP framework for point processes as follows. 

For notational simplicity, we denote the $d$-dimensional marked temporal sapce as $\mathcal{X} \coloneqq [0, T) \times \mathcal{M} \subset \mathbb{R}^d$. Let data tuple of the current event be $x \coloneqq (t, m) \in \mathcal{X}$, and the data tuple of an arbitrary past event be $x^\prime \coloneqq (t^\prime, m^\prime) \in \mathcal{X}$ for any $t^\prime < t$. 
For the $k$-th \emph{attention head}, 
we first score the current event against its past event 
using score function $\nu^{(k)}: \mathcal{X} \times \mathcal{X} \rightarrow \mathbb{R}^+$. 
For the event $x$, the score $\nu^{(k)}(x, x^\prime)$ determines how much \emph{attention} to place on the past event $x^\prime$ as we encode the history information.
More details about the score formulation will be presented in Section~\ref{sec:score}.
The normalized score $\widetilde\nu^{(k)}(x, x^\prime) \in [0, 1]$ for the event $x$ and $x^\prime$ is obtained by employing the softmax function over the score, which is defined as
\begin{equation}
    \widetilde\nu^{(k)}(x, x^\prime) = \frac{
    \nu^{(k)}(x, x^\prime)}{
    \sum_{t_i < t} \nu^{(k)}(x, x_i)},~k = 1,\dots, K,
    \label{eq:score}
\end{equation}   
Then we map past events to the \emph{value} embedding space via $\varphi^{(k)}: \mathcal{X} \rightarrow \mathbb{R}^p$, where $p$ is the dimension of the value embedding.
Here the value embedding is a linear transformation of the event's data tuple, i.e., $\varphi^{(k)}(x) = W_v^{(k)} x \in \mathbb{R}^p$, where $W_v^{(k)} \in \mathbb{R}^{p \times d}$ is the weight matrix. 
Therefore, the $k$-th attention head $\boldsymbol{h}^{(k)}(x) \in \mathbb{R}^p$ for the event $x$ can be obtained by multiplying each value embedding by the score and adding them up, which is formally defined as
\begin{equation}
    h^{(k)}(x) = \sum_{t_i < t} \widetilde\nu^{(k)}(x, x_i) \varphi^{(k)}(x_i),~k = 1,\dots, K,
    \label{eq:attention}
\end{equation} 
Note that events $x, x_i$ are analogous to the \emph{query} and the $i$-th \emph{key}, the embedding of the $i$-th event $\varphi^{(k)}(x_i)$ is analogous to the \emph{value} in the attention mechanism. 
The multi-head attention $\boldsymbol{h}(x) \in \mathbb{R}^{Kp}$ is the concatenation of $K$ single attention heads:
\[
    \boldsymbol{h}(x) = \text{concat}\left(h^{(1)}(x), \dots, h^{(K)}(x)\right).
\]
We highlight that the attention $\boldsymbol{h}(x)$ is able to ``emphasize'' (or ``de-emphasize'') events, which are most (or least) influential in their future, by directly assigning them larger (smaller) scores. 
In comparison, RNN-based models pass the history information sequentially via a hidden state, where the recent memory will override the long-term memory. This has led RNNs to ``overemphasize'' the recent events and fail to capture the events' influences in the distant past.  

Follow the similar idea of \cite{Mei2017}, we consider a non-linear transformation of the multi-head attention $\boldsymbol{h}(x)$ as the historical information before event $x$, the conditional intensity function $\lambda$ can be specified as:
\begin{equation}
    \lambda(x | \boldsymbol{h}(x)) = \underbrace{\mu(x)}_{\text{base intensity}} + \underbrace{g\left(\boldsymbol{h}(x)^\top W + b\right)}_{\text{triggering effect}},
    \label{eq:lambda}
\end{equation}
where $W \in \mathbb{R}^{Kp}, b \in \mathbb{R}$ are the weight matrix and the bias term, where $g: \mathbb R \rightarrow \mathbb R^+$ is a monotonically increasing function, and here we choose the function $g(x):= \text{softplus}(x) = \log(1 + e^x) > 0$ is a smooth approximation of the ReLU function, which ensures the intensity strictly positive at all times when an event could possibly occur and avoid infinitely bad log-likelihood. The $\mu(x) > 0$ is the base intensity, which can be estimated from the data.

\vspace{-0.1in}
\subsection{Score function via deep Fourier kernel}
\label{sec:score}
\vspace{-0.1in}

As introduced in the previous section, the score function $\nu$ directly quantifies how likely one event is triggered by the other in a sequence, which plays a similar role as the triggering function in Hawkes processes defined in \eqref{eq:hawkes}. 
%
%
Usually, the \emph{dot-product score} has been widely used in most attention models.
Specifically, two points $x, x' \in \mathbb{R}^d$ are first projected onto another space via $W_u x $ and $W_u x'$ (the so-called \emph{key embeddings}), where $W_u \in \mathbb{R}^{r \times d}$ is a linear mapping and $r$ is the dimension of key embeddings. 
Then the score is obtained by computing 
$
x^\top W_u^\top W_u x',
$
which essentially is their inner product in the embedding space.
However, for some real applications, inner product or Euclidean distance may be limited when the triggering effects between events are non-homogeneous. 

\begin{figure}[!h]
\vspace{-0.1in}
\centering
\includegraphics[width=.95\linewidth]{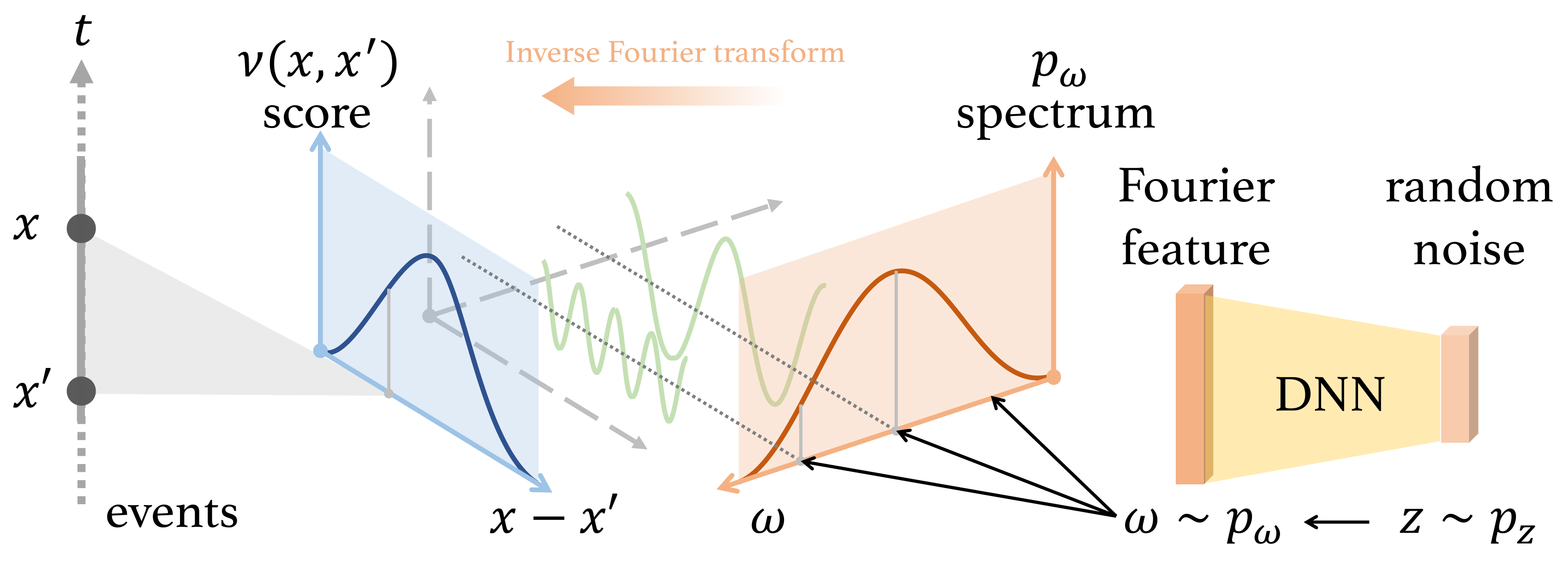}
\caption{\small{An illustration for the Fourier kernel score. The score is computed via an inverse Fourier transform, where the distribution of Fourier features is represented by a deep neural network.}}
\label{fig:score-illustration}
\vspace{-0.1in}
\end{figure}

We aim to find a more representative score function without specifying any parametric form to capture the complex interactions between events, which goes beyond the linear assumption on the dot-product score. 
To this end, we propose a novel deep Fourier kernel as the score function in the attention mechanism, where the key embedding $W_u x$ is substituted with the kernel-induced feature mapping $\Phi(x)$.  
As shown in Figure~\ref{fig:score-illustration}, these feature mappings are randomly sampled from an optimal high-dimensional power spectrum. 
The optimal spectrum (the distribution of power) is represented by a deep neural network. The network's inputs are random normal noises, and the outputs are Fourier features sampled from the optimal spectrum. 

Formally, this score formulation relies on Bochner’s Theorem \citep{rudin1962fourier}, which states that any bounded, continuous and shift-invariant kernel is a Fourier transform of a bounded non-negative measure:
\begin{theorem}[Bochner \citep{rudin1962fourier}]
\label{thm1}
A continuous kernel of the form $\nu(x,x^\prime)=\kappa(x-x^\prime)$ defined over a locally compact set $\mathcal{X} \subset \mathbb{R}^d$ is positive definite if and only if $g$ is the Fourier transform of a non-negative measure:
\begin{equation}
    \nu(x, x^\prime) = \kappa(x - x^\prime) =\int_{\Omega}p(\omega)e^{jw^\top (x - x^\prime)}d\omega,
    \label{eq:score-def-1}
\end{equation}
where $p$ is a non-negative measure, $\Omega$ is the Fourier feature space, and kernels of the form $\nu(x,x^\prime)$ are called shift-invariant kernel. 
\end{theorem}
If a shift-invariant kernel $\kappa(\cdot)$ is properly scaled such that $\kappa(0) = 1$, Bochner’s theorem guarantees that its Fourier transform $p(\omega)$ is a proper probability distribution. 

Suppose an optimal spectrum that best describes how the ``energy'' of events' interaction in each attention head is distributed with Fourier features. 
Here we assume $p_\omega^{(k)}$ is the optimal distribution of Fourier features $\omega \in \Omega \subset \mathbb{R}^r$ in the $k$-th attention head, where $r$ is dimension of Fourier features. 
We also substitute $\exp\{jw^\top (x - x^\prime)\}$ with a real-valued feature mapping, such that the probability distribution $p_\omega$ and the kernel $\nu$ are real \citep{Rahimi2008}. 
We, therefore, obtain a score formulation of the $k$-th attention head in \eqref{eq:score} between two events $x, x^\prime \in \mathcal{X} \subset \mathbb{R}^d$ that satisfies these conditions as the following proposition (see proof in Appendix~\ref{append:proof-prop-1}):
\begin{prop}[Score function via Fourier kernel embedding]
\label{prop:score-reformulation}
Let the score $\nu^{(k)},~k=1,\dots,K$ be a continuous real-valued shift-invariant kernel and $p_\omega^{(k)}$ be a probability distribution, we have the following definition:
\begin{equation}
    \nu^{(k)}(x, x^\prime) \coloneqq \mathbb{E}
    \big [ \phi^{(k)}_{\omega}(x) \cdot \phi^{(k)}_{\omega}(x^\prime) \big ],
    \label{eq:score-def-2}
\end{equation}
where $\phi^{(k)}_{\omega}(x) \coloneqq \sqrt{2} \cos(\omega^\top W_u^{(k)} x + b_u),$ and $W_u^{(k)} \in \mathbb{R}^{r \times d}$ is a linear mapping.
These Fourier features $\omega \in \Omega \subset \mathbb{R}^r$ are sampled from $p_{\omega}^{(k)}$ and $b_u$ is drawn uniformly from $[0, 2\pi]$.
\end{prop}
We can conclude from the proposition that (1) the score function is defined by the optimal spectrum $p^{(k)}_\omega$ and the weight $W_u^{(k)}$.  
Here $W_u^{(k)} x$ resembles the \emph{key embedding} in the dot-product score, 
which projects event $x$ to a high-dimensional embedding space;
(2) this representation enables us to conveniently estimate the score from samples, i.e.,
\begin{equation}
    \nu^{(k)}(x, x') \approx \frac{1}{D} \sum_{j=1}^D \phi^{(k)}_{\omega_j}(x) \cdot \phi^{(k)}_{\omega_j}(x^\prime) = \Phi^{(k)}(x)^\top \Phi^{(k)}(x^\prime),
    \label{eq:score-def-3}
\end{equation}
where $\omega_j,j=1,\dots,D$ are $D$ Fourier features sampled from the distribution $p_\omega^{(k)}$.
The vector 
\[
\Phi^{(k)}(x) \coloneqq [ \phi^{(k)}_{\omega_1}(x), \dots, \phi^{(k)}_{\omega_D}(x) ]^\top,
\]
can be viewed as the approximation of the kernel-induced feature mapping for the score function. 

In the following proposition, we will show this empirical estimation converges uniformly over a compact domain $\mathcal{X}$ as $D$ grows and is a lower variance approximation to \eqref{eq:score-def-2} (see the proof in Appendix~\ref{append:proof-prop-2}):
\begin{prop}[Concentration of empirical scores]
\label{prop:score-convergence}
Assume $\sigma_p^2 = \mathbb{E}_{\omega \sim p^{(k)}_\omega} [\omega^\top \omega] < \infty$ and $\mathcal{X} \subset \mathbb{R}^d$. Let $R$ denote the radius of the Euclidean ball containing $\mathcal{X}$, then for the kernel-induced feature mapping $\Phi^{(k)}$ defined in \eqref{eq:score-def-3}, we have
\begin{equation}
\begin{aligned}
    &\mathbb{P}\left\{\underset{x, x^\prime \in \mathcal{X}}{\sup} \left | \Phi^{(k)}(x)^\top \Phi^{(k)}(x^\prime) - \nu^{(k)}(x, x^\prime) \right | \ge \epsilon \right\} \\
    \le & \left(\frac{48 R \sigma_p}{\epsilon}\right)^2\exp\left\{ - \frac{D \epsilon^2}{4(d+2)} \right\}.
    \label{eq:convergence}
\end{aligned}
\end{equation}
\end{prop}
The proposition guarantees that a good estimate of the score function can be found, with high probability, by sampling a finite number of Fourier features. In particular, for an absolute error of at most $\epsilon$, the number of samples needed is on the order of $D = O(d\log(R\sigma_p/\epsilon) / \epsilon^2)$, which grows linearly as data dimension $d$ increases. 

\vspace{-0.1in}
\subsection{Fourier feature generator} 
\vspace{-0.1in}

To represent the distribution $p_{\omega}^{(k)}$ over Fourier feature $\omega$, we define a prior (generator) on an input noise variable $z \sim p_z$, then represent a mapping to feature space as $G: \mathbb{R}^q \rightarrow \mathbb{R}^r$ as shown in Figure~\ref{fig:score-illustration}, where $G$ is a differentiable function characterized by a deep neural network with parameters $\theta^{(k)}$ and $q$ is the dimension of the noise, such that roughly speaking the distribution functions are the same $p_{\omega}^{(k)}\approx G(z)$.
Note that the score function's representative power is jointly decided by the generator's parameters and the weight matrix of the key embedding. 

Figure~\ref{fig:score-transformation} gives an intuitive example of representing the intensity of events using our DAPP with two attention heads ($K=2$).
Here, we choose $q = r = 2$ to visualize the noise prior and the optimal spectrums in a 2D space for ease of presentation. The optimal spectrum learned from data in each attention head uniquely specifies a score function, which is capable of capturing various types of non-linear triggering effects. Unlike Hawkes processes, underlying long-term influences of some events, in this case, can be preserved in the intensity function.
As shown in Figure~\ref{fig:score-visualization}, we present two examples of pairwise scores calculated by the proposed Fourier score and dot-product score under the same architecture, respectively, which enables a visual comparison. To make these two methods comparable, we trained two models using the same synthetic data set, and its exact triggering function is also provided as the ``ground truth''. This ablation study confirms that our Fourier score in a single attention head is expressive enough to accurately capture the triggering effects.

\begin{figure}
\centering
\includegraphics[width=1.\linewidth]{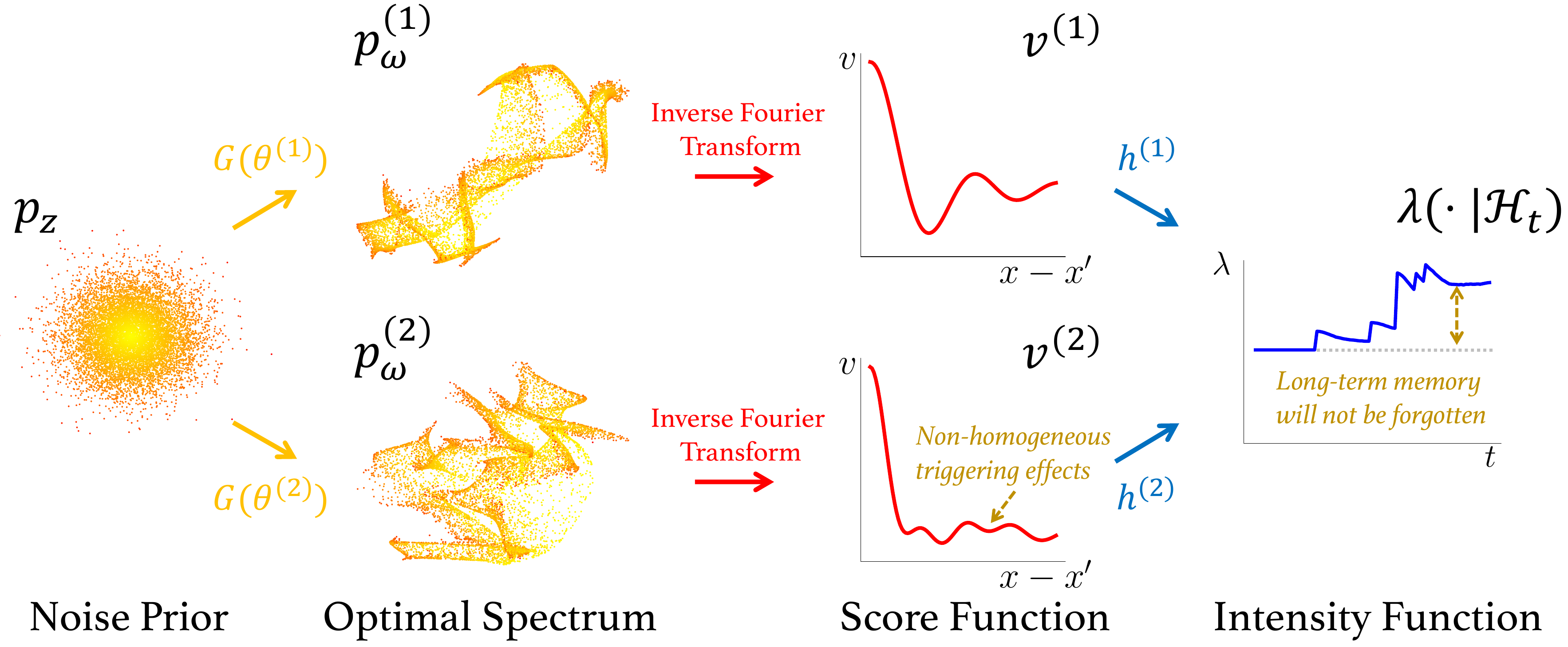}
\caption{\small{A real example of optimal spectrums, score function, and corresponding intensity function learned from DAPP. Here, the DAPP is trained using real 911 calls-for-service data with recorded time only in 2017 provided by Atlanta Police. There are $10,000$ Fourier features sampled from the optimal spectrums being used to reconstruct score functions. The right-most sub-figure represents the intensity of a 911 call sequence reported in a single day at beat 702. We can see that Fourier kernel score is able to capture non-homogeneous triggering effects of events and long-term memory will also not be forgotten in this case.}}
\label{fig:score-transformation}
\end{figure}

\begin{figure*}[!h]
\vspace{-0.1in}
\begin{center}
  \subfigure[Fourier kernel]
  {\includegraphics[width=.32\linewidth]{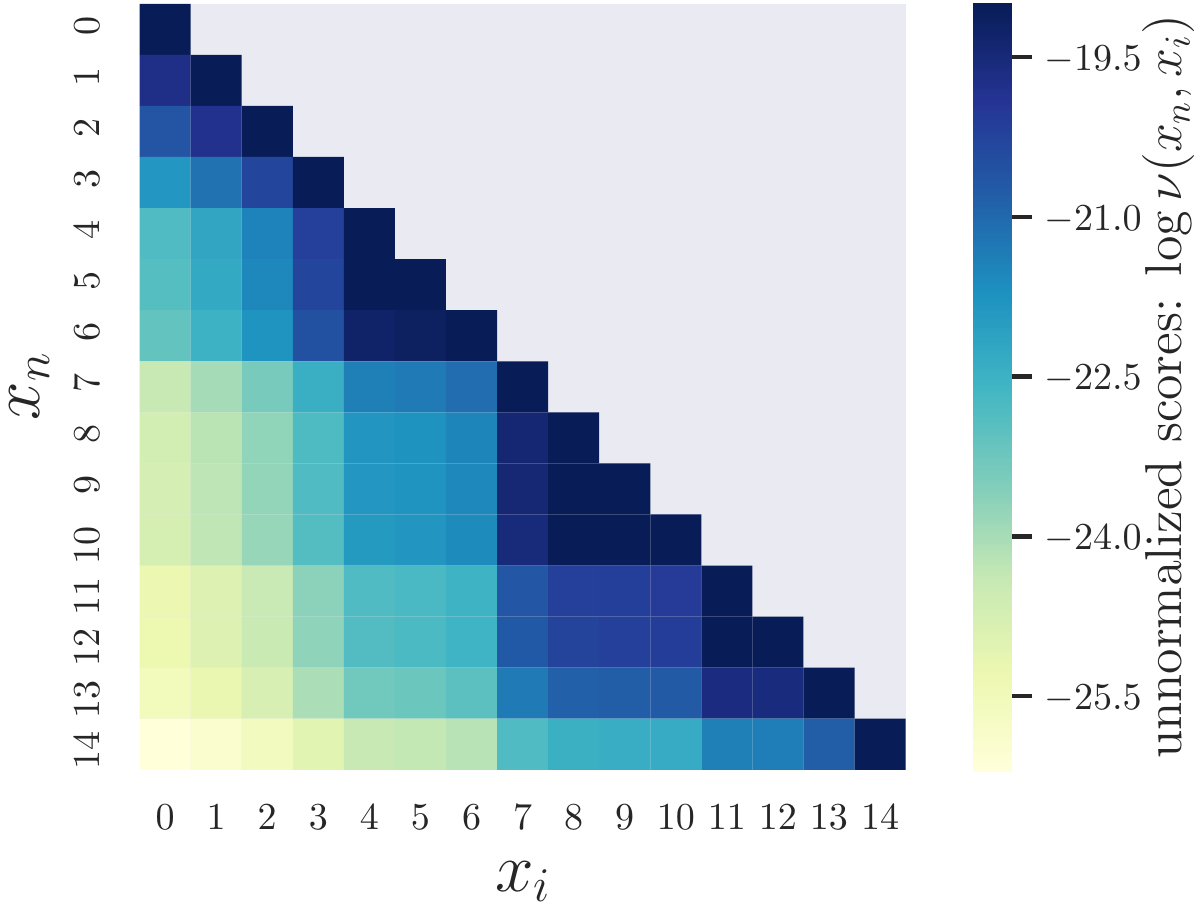}}
  \subfigure[dot-product]
  {\includegraphics[width=.32\linewidth]{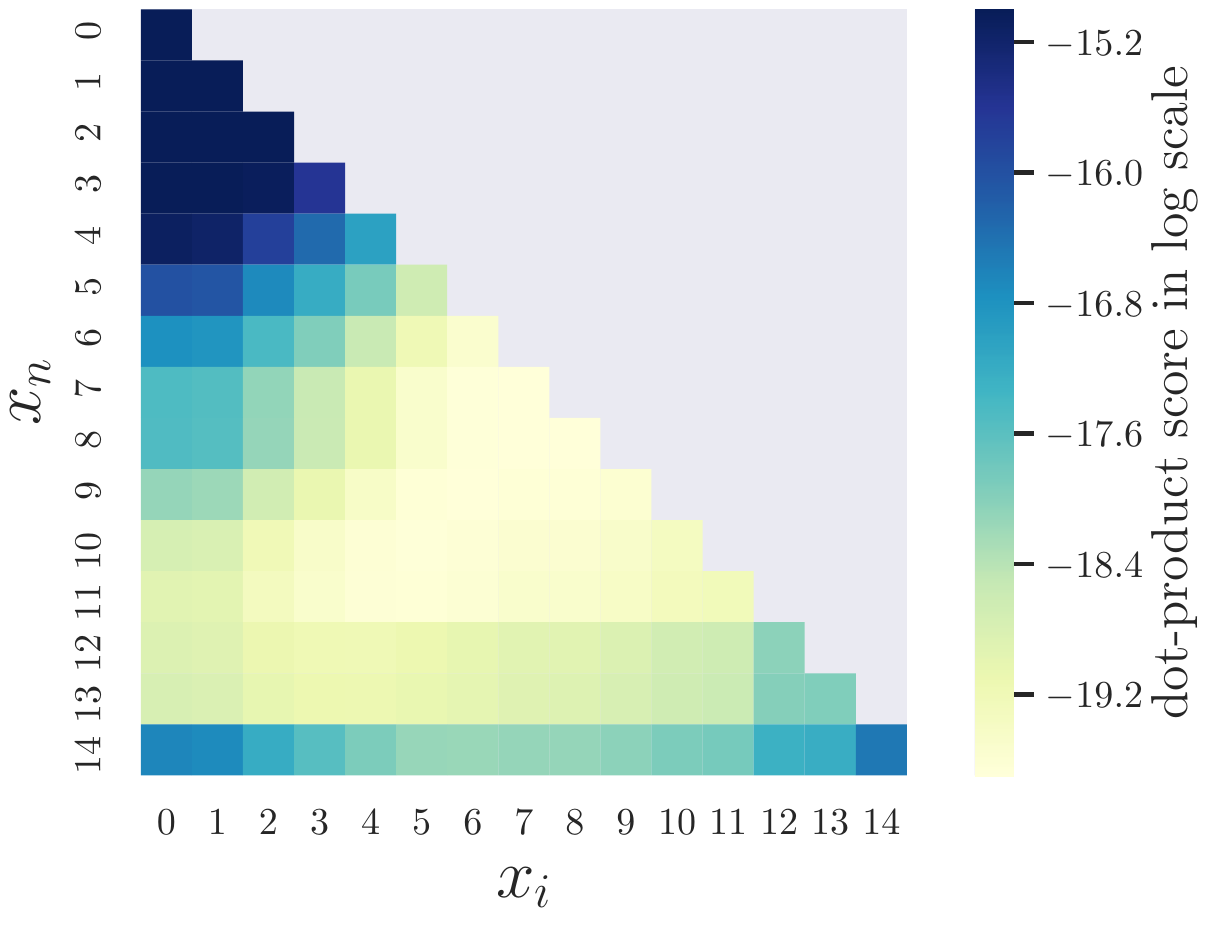}}
  \subfigure[truth]
  {\includegraphics[width=.32\linewidth]{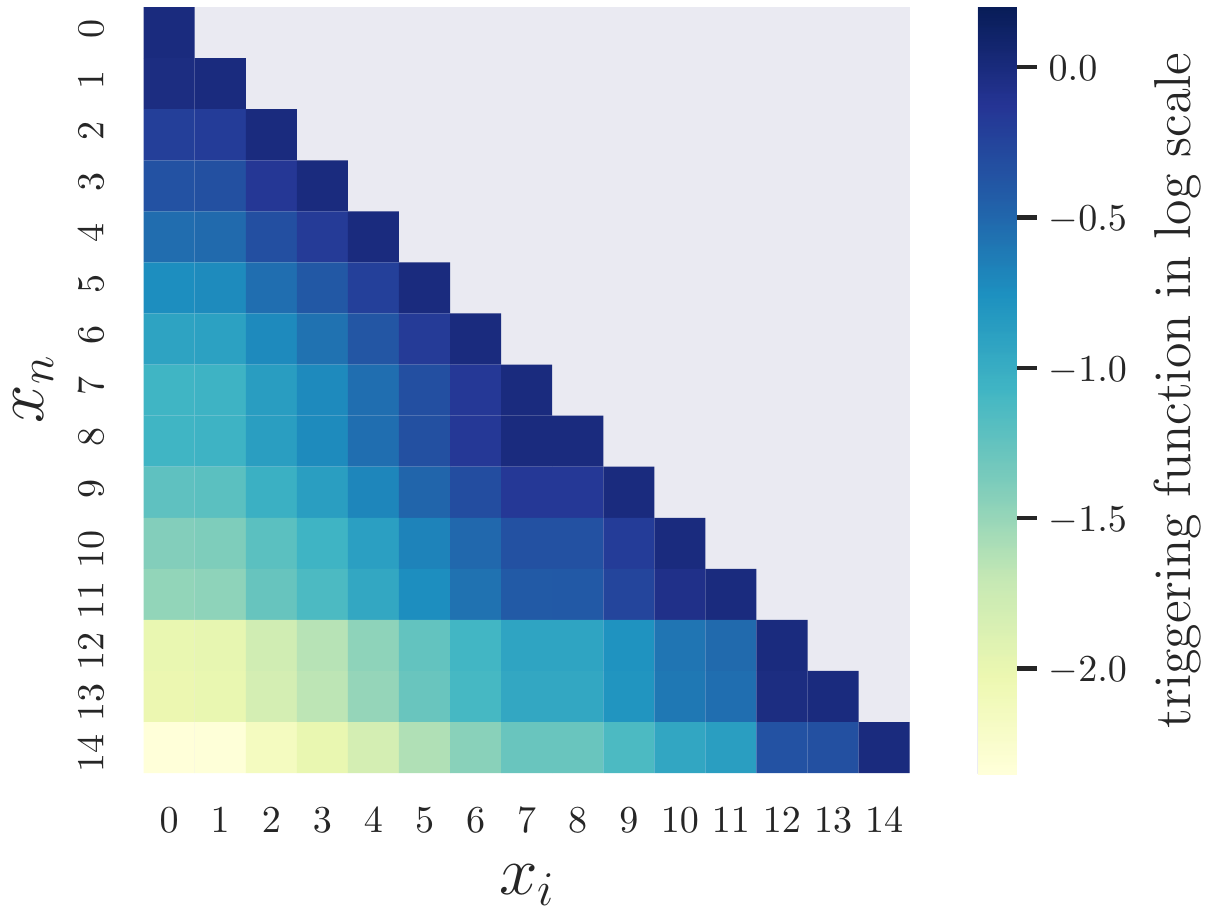}}
  \vspace{-0.1in}
  \caption{\small{Pairwise scores between events learned from DAPP using a synthetic Hawkes process data set. The $x_n$ denotes the current event and $x_i$ denotes past events, where $t_i < t_n$. The color of the entry at $n$-th row and $i$-th column in these figures indicates: (a) Fourier kernel scores; (b) dot-product score (using the same architecture by only substituting Fourier kernel score with dot-product score); (c) true triggering effects evaluated by triggering function $g(t, t_i) = \alpha \exp \{ - \beta (t - t_i) \}$ (does not exactly correspond to the score, but reveals some key facts on the correlation of events, e.g., exponential decaying over time).}}
  \label{fig:score-visualization}
\end{center}
\vspace{-0.2in}
\end{figure*}

\vspace{-0.1in}
\subsection{Online attention for streaming data}
\label{sec:online-attention}
\vspace{-0.1in}

The attention calculation may be computationally intractable for streaming data since the number of past events would overgrow as time goes on. Here, we propose an adaptive online attention algorithm to address this issue. Only a fixed number of ``important'' historical events with high average scores will be remembered for the attention calculation in each attention head. The procedure for collecting ``important'' events in each attention head is demonstrated as follows.

When the $i$-th event occurs, for a past event $x_j, t_j < t_i$ in $k$-th attention, we denote the set of its score against future events as $\mathscr{S}_{j}^{(k)} \coloneqq \{\widetilde\nu^{(k)}(x_i, x_j)\}_{i:t_j \le t_i}$. Then the average score of the event $x_j$ can be computed by 
\[
\bar\nu_{j}^{(k)} = (\sum\nolimits_{s \in \mathscr{S}_{j}^{(k)}} s)/|\mathscr{S}_{j}^{(k)}|,
\] 
where $|A|$ denotes the number of elements in set $A$. Hence, a recursive definition of the set of active events $\mathscr{A}_{i}^{(k)}$ in the $k$-th attention head up until the occurrence of the event $x_i$ is written as: 
\begin{align*}
    & \mathscr{A}_{i}^{(k)} = \mathcal{H}_{t_{i+1}},~\forall i \le \eta,\\
    & \mathscr{A}_{i}^{(k)} = \mathscr{A}_{i-1}^{(k)} \cup 
    \{x_i\}
    \setminus \underset{j: t_j < t_i}{\arg \min} \left\{ \bar\nu_{j}^{(k)} \right\},~\forall i > \eta,
\end{align*}
where $\eta$ is the maximum number of events we want to remember. 
The exact event selection is carried out by Algorithm~\ref{alg:online-attention} shown in Appendix~\ref{append:learning}. 
To perform the online attention, we substitute $\mathcal{H}_{t_n}$ in \eqref{eq:score} and \eqref{eq:attention} with $\mathscr{A}_{n}^{(k)}$ for all attention heads.

\vspace{-0.1in}
\subsection{Learning and simulation}
\label{sec:learning}
\vspace{-0.1in}

The proposed model is jointly parameterized by $\boldsymbol{\theta}=\{W, b, \{\theta^{(k)}, W^{(k)}_u, W^{(k)}_v\}_{k=1,\dots,K}\}$, which can be learned via \emph{maximum likelihood estimation} using the stochastic gradient descent. The log-likelihood function of the model can be obtained by substituting \eqref{eq:lambda} into \eqref{eq:log-likelihood} defined in Section~\ref{sec:point-process}. The exact learning algorithm is carried out by Algorithm~\ref{alg:learning} shown in Appendix~\ref{append:learning}. 

A default way to generate events from a point process is to use the thinning algorithm \citep{Daley2008, Gabriel2013}. However, the vanilla thinning algorithm suffers from low sampling efficiency as it needs to sample in the space $\mathcal{X}$ uniformly with the upper limit of the conditional intensity $\bar{\lambda}$ and only very few candidate points will be retained in the end. 
%
To improve sampling efficiency, we use an efficient thinning algorithm summarized in Algorithm~\ref{alg:thinning}, Appendix~\ref{append:learning}. The ``proposal'' density is a non-homogeneous MTPP, whose intensity function is defined from the previous iterations. This analogous to the idea of rejection sampling \citep{Ogata1981}.

\vspace{-0.1in}
\section{Experiments}
\label{sec:experiments}
\vspace{-0.1in}


\begin{figure*}[!h]
\vspace{-0.1in}
  \centering
  \subfigure[Hawkes]{\includegraphics[width=.24\linewidth]{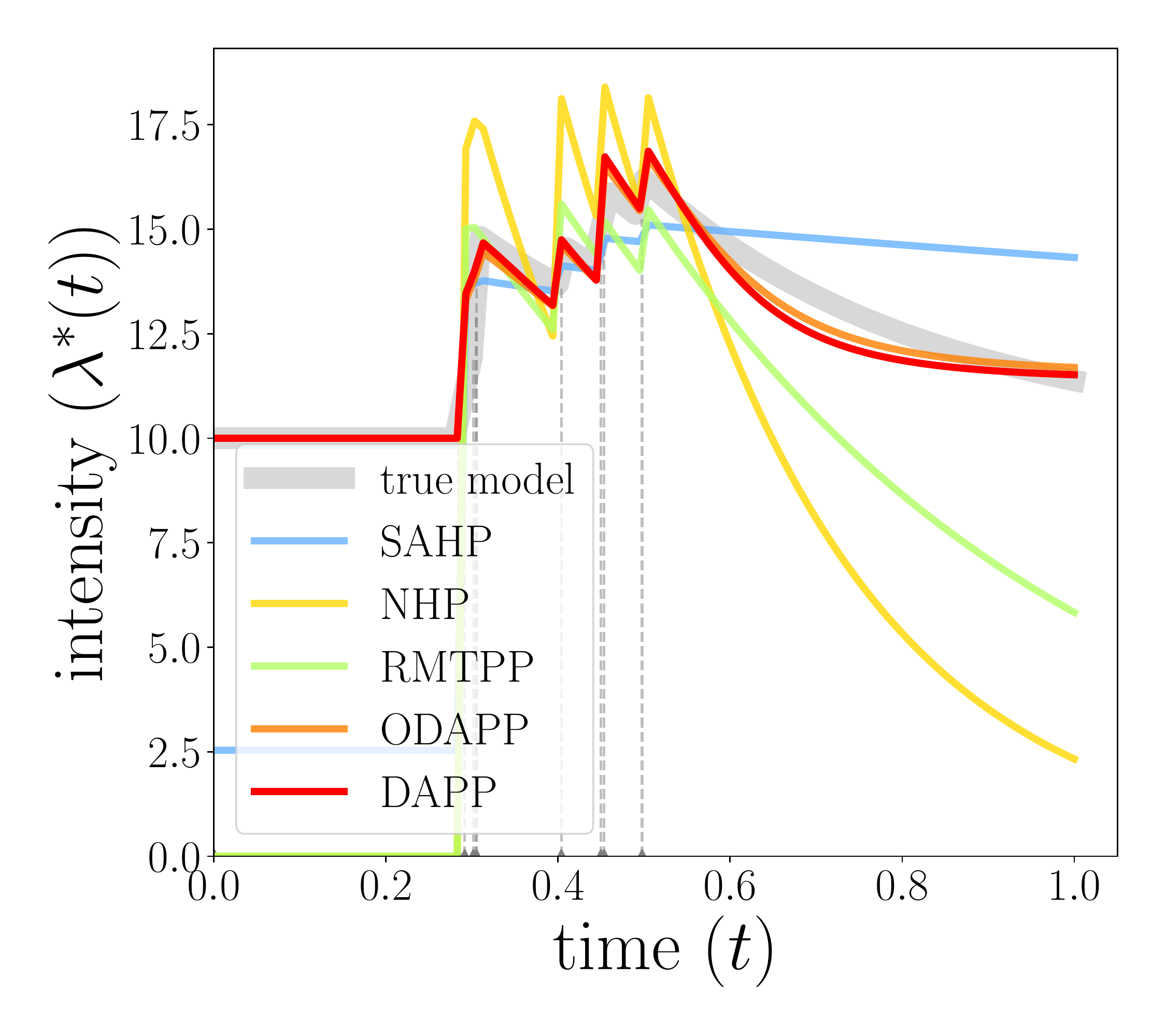}}
  \subfigure[self-correction]{\includegraphics[width=.24\linewidth]{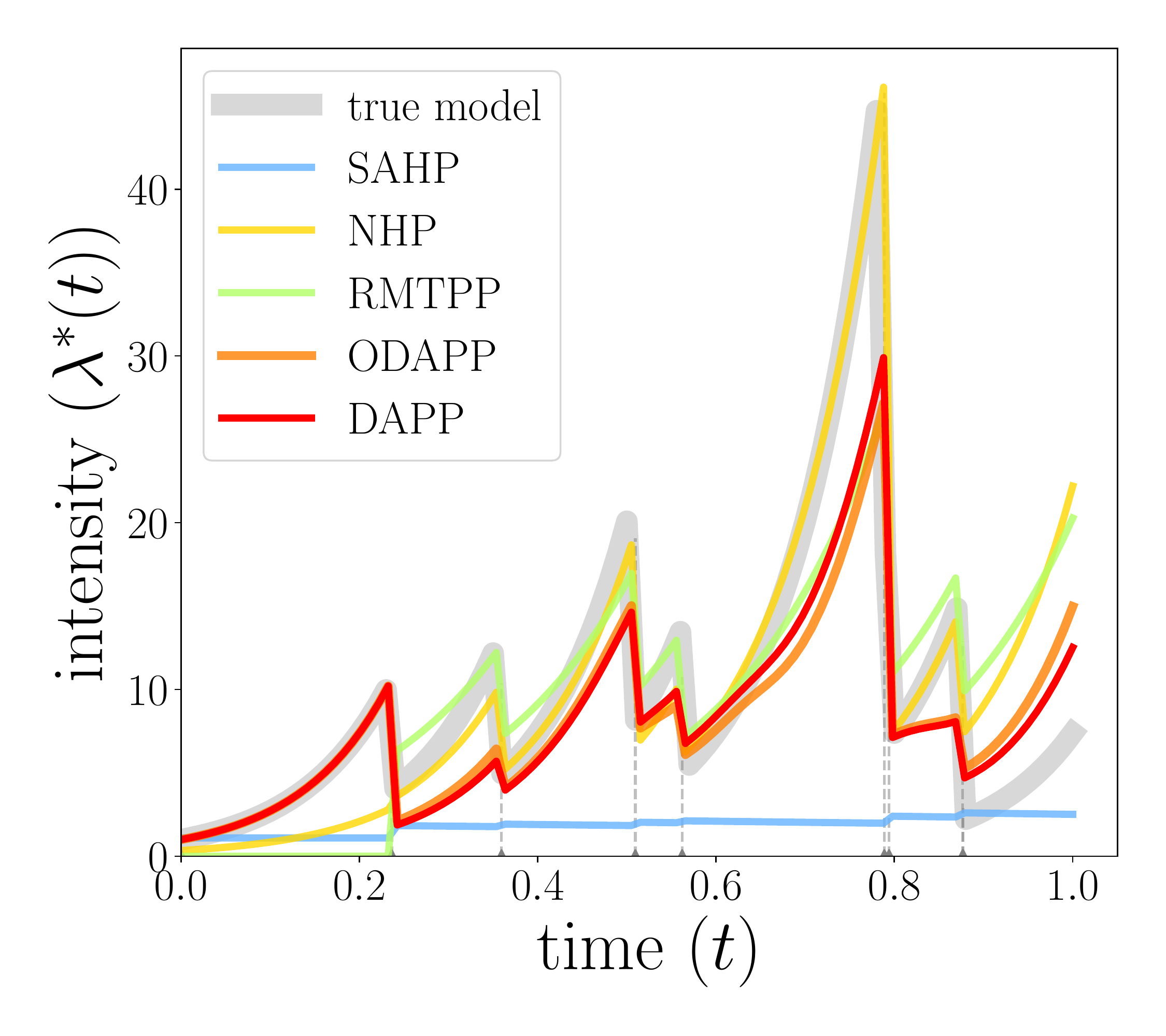}}
  \subfigure[non-homogeneous $1$]{\includegraphics[width=.24\linewidth]{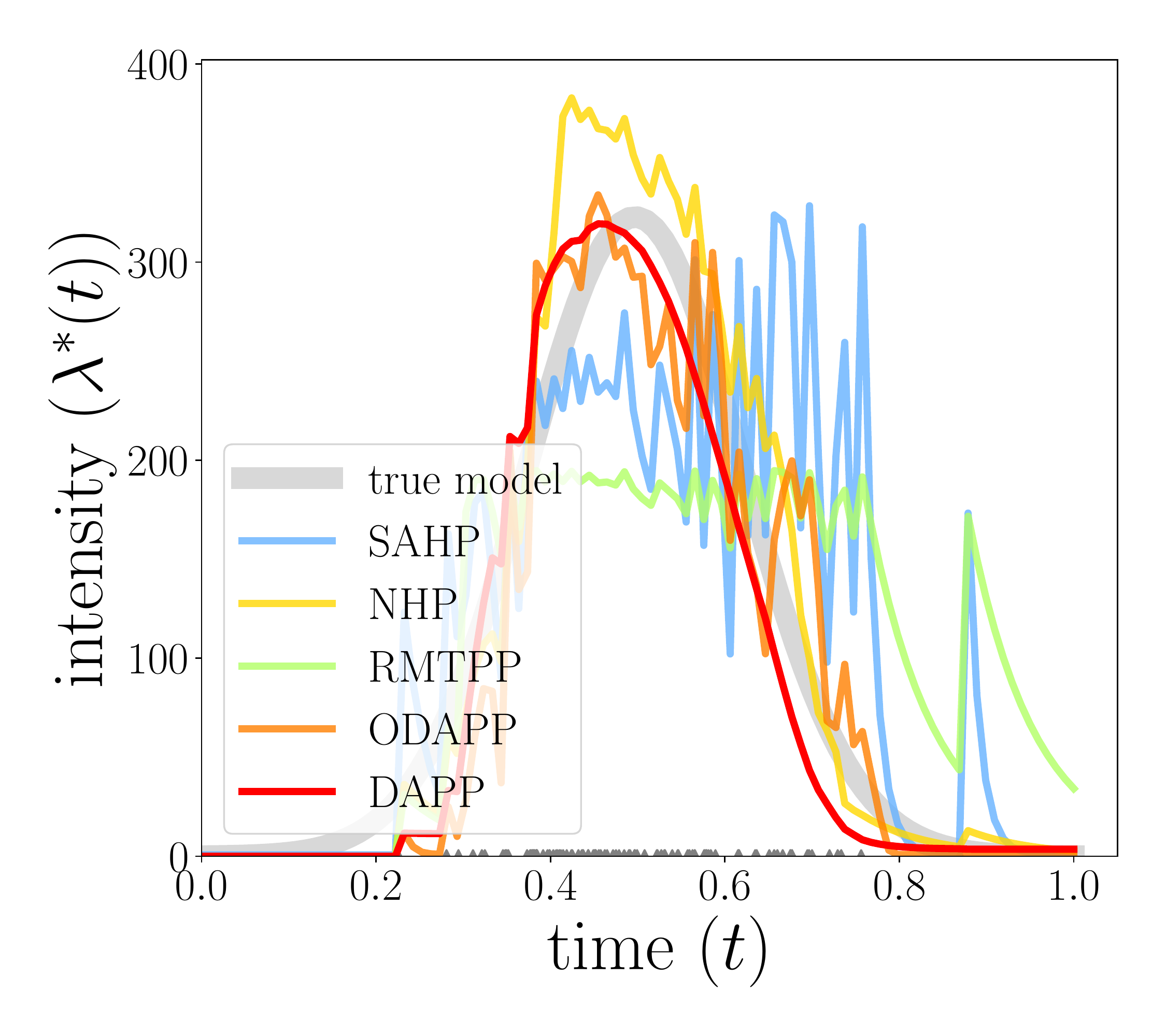}}
  \subfigure[non-homogeneous $2$]{\includegraphics[width=.24\linewidth]{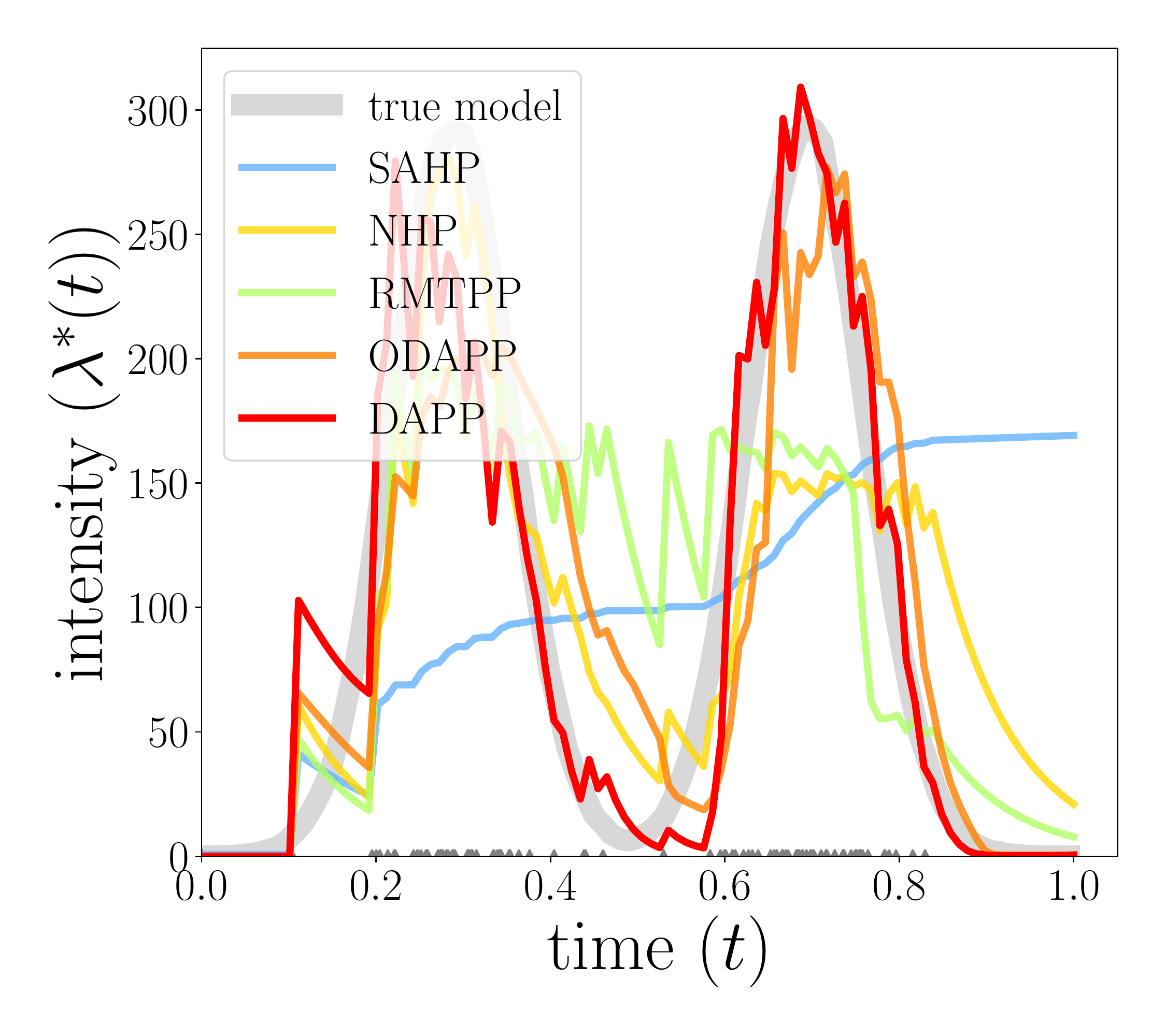}}
  \vspace{-0.1in}
  \caption{\small{Conditional intensity function estimated from synthetic data sets. Triangles at the bottom of each panel represent events. The ground truth of conditional intensities is indicated by the grayline.}}
  \label{fig:sim-res-intensity}
  \vspace{-0.1in}
\end{figure*}

\begin{figure*}[!h]
\vspace{-0.1in}
  \centering
  \subfigure[stock]{\includegraphics[width=.24\linewidth]{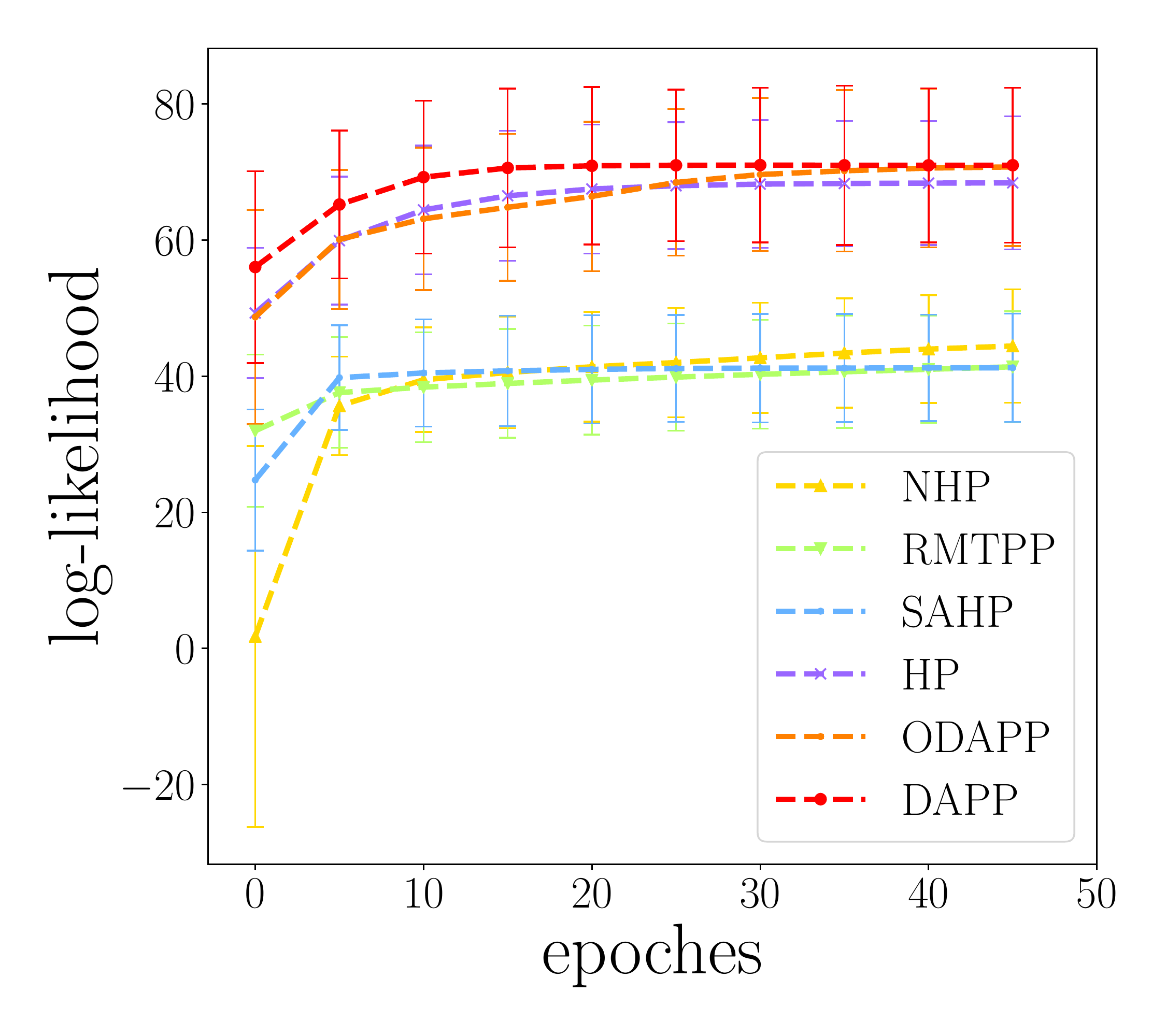}}
  \subfigure[MIMIC-III]{\includegraphics[width=.24\linewidth]{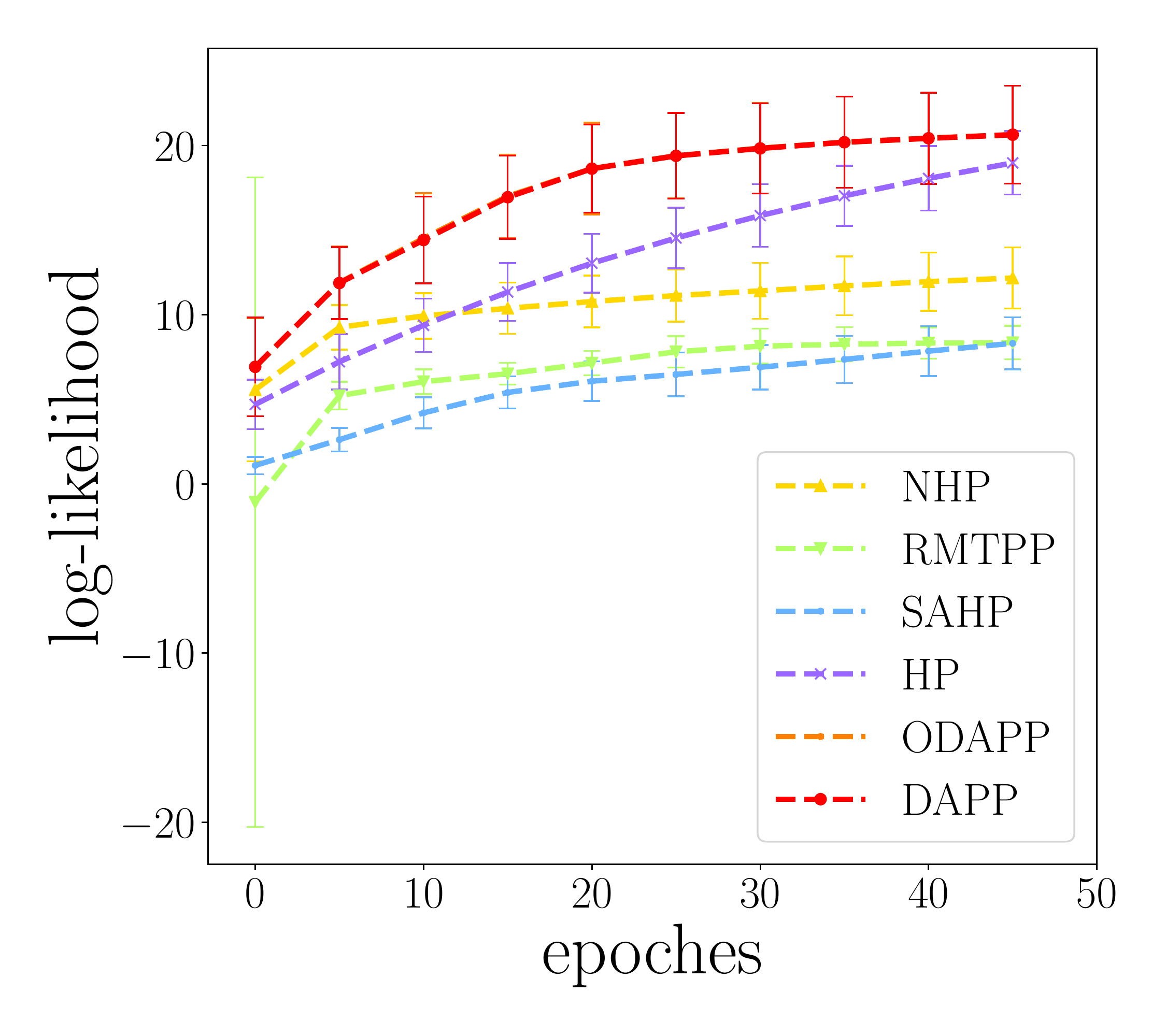}}
  \subfigure[meme]{\includegraphics[width=.24\linewidth]{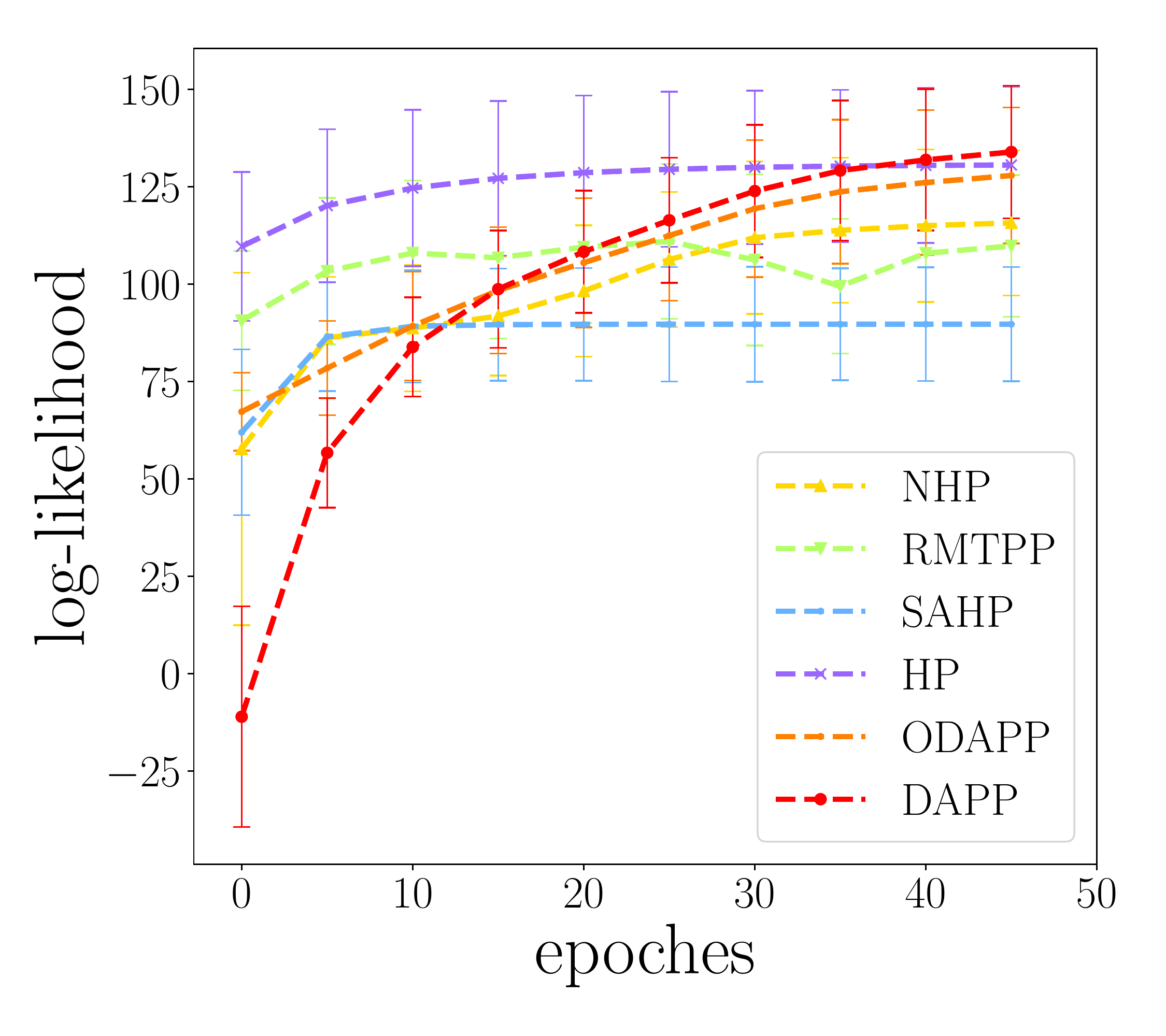}}
  \subfigure[traffic]{\includegraphics[width=.24\linewidth]{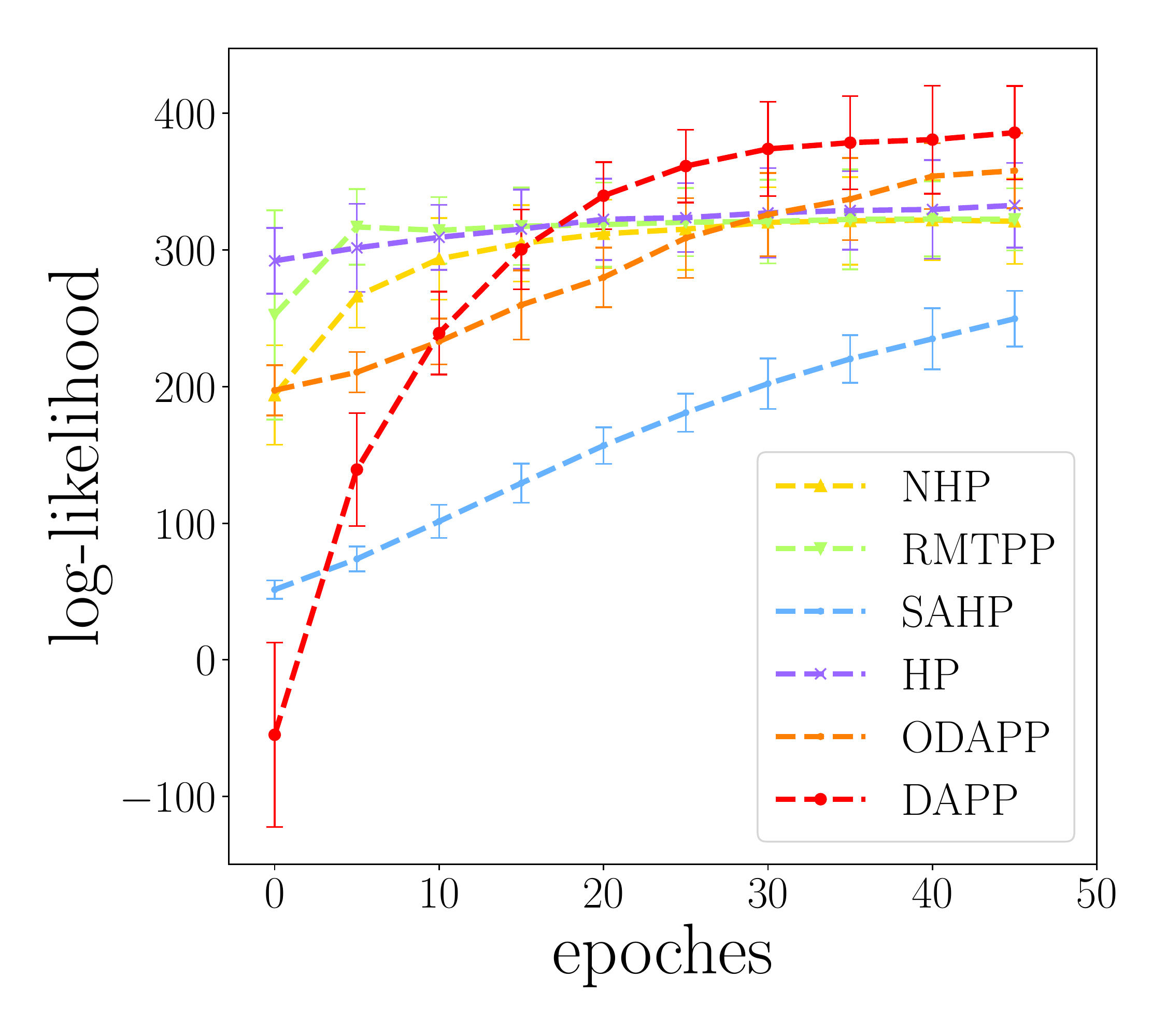}}
  \vspace{-0.1in}
  \caption{\small{The average log-likelihood of real data sets versus training epochs. For each real data set, we evaluate performance of the five methods according to the final log-likelihood averaged per event calculated for the test data.}}
  \label{fig:real}
  \vspace{-0.1in}
\end{figure*}

In this section, we conduct experiments on four synthetic data sets and four large-scale real-world data sets. 
We compare our DAPP and its online version (ODAPP) with the other four baselines by evaluating the mean square intensity-recovering error and the likelihood value, which have been widely adopted in the related works \citep{Mei2017, Omi2019, Zhang2019}. 
The implementation details of baselines are discussed in Section~\ref{sec:baseline-description}.
We describe the experiment configurations as follows: we consider two attention heads ($K=2$) in DAPP and ODAPP, where the Fourier feature generator $\theta^{(k)}$ of the $k$-th head is characterized by a fully-connected neural network with three hidden layers, where the widths of each layer are 128, 256, and 128, respectively. 
To learn DAPP and its associated optimal spectrums more efficiently, we adopt the stochastic gradient descent method and only sample a few points of Fourier features ($D=20$) for each mini-batch. 
For accurate intensity recovery, a more significant number of Fourier features ($D=10,000$) will be sampled in a bid to reconstruct a high-resolution optimal spectrum.
Besides, there is only a 50\% number of events are retained for training ODAPP, i.e., $\eta = 0.5 n$, where $n$ is the maximum length of sequences in each data set.


\begin{table*}
\caption{the mean square error of recovering the intensity.}
\label{tab:pred-error}
\vspace{-0.1in}
\begin{center}
\begin{small}
\begin{sc}
\resizebox{1.\textwidth}{!}{%
\begin{tabular}{lccccccccccr}
\toprule[1pt]\midrule[0.3pt]
Data set & HP & SAHP & NHP & RMTPP & {DAPP+dot-prod} & {ODAPP+dot-prod} & {DAPP+NN} & {ODAPP+NN} & DAPP & ODAPP \\
\midrule
hawkes & 0.031& 18.3 & 49.9 & 35.9 & {15.3} & {19.3} & {1.221} & {1.893} & 0.258 & \textbf{0.166} \\
self-correction & 74.3 & 130.8 & 25.8 & 36.1 & {117.4} & {133.3} & {47.2} & {51.3} & \textbf{21.8}& 27.3 \\
non-homo $1$ & 1672.3 & 7165.5& 1431.6&6852.3 & {6201.5} & {7014.8} & {972.5} & {1124.1} & \textbf{605.7}& 1511.8 \\
non-homo $2$ & 3210.3 & 9858.9& 2063.1& 3854.8 & {9812.7} & {9733.1} & {1449.5} & {1722.7} & \textbf{1351.4} & 1527.9 \\
\midrule[0.3pt]\bottomrule[1pt]
\end{tabular}
}
\end{sc}
\end{small}
\end{center}
\vspace{-0.2in}
\end{table*}

\begin{table}[!b]
\vspace{-0.2in}
\caption{the average log-likelihood.}
\label{tab:loglik}
\begin{center}
\begin{small}
\begin{sc}
\resizebox{0.48\textwidth}{!}{%
\begin{tabular}{lccccccr}
\toprule[1pt]\midrule[0.3pt]
Data set & HP & SAHP & NHP & RMTPP & DAPP & ODAPP \\
\midrule
hawkes   & 22.0 & 20.8 & 20.0 & 19.7 & \textbf{21.2} & 21.1  \\
self-correction & 3.9 & 3.5& 5.4& 6.9 & 7.1&\textbf{7.1} \\
non-homo $1$   & 437.8 & 432.4& 445.6&443.1 & 442.3& \textbf{457.0} \\
non-homo $2$   & 399.4 & 364.3& 410.1& 405.1  & \textbf{428.3} & 420.1     \\
mimic-iii   & 17.1  & 11.7& 14.4 &8.7 &\textbf{21.5} & 21.2\\
stock   & 66.3   & 43.1& 43.4& 44.0 & \textbf{72.9} & 72.9\\
meme   & 129.8   & 84.0& 113.4& 106.0 & \textbf{131.0}&  128.5     \\
traffic & 313.8   & 326.7& 324.4& 339.2 & \textbf{458.5}& 387.2\\
\midrule[0.3pt]\bottomrule[1pt]
\end{tabular}
}
\end{sc}
\end{small}
\end{center}
\vspace{-0.1in}
\end{table}

\vspace{-0.1in}
\subsection{Baseline methods}
\label{sec:baseline-description}
\vspace{-0.1in}

{\it Recurrent Marked Temporal Point Process} (\texttt{RMTPP}) assumes the following form for the conditional intensity function $\lambda^*$ in point processes, denoted as
$
\lambda^*(t) = \exp{\big({\boldsymbol{v}}^{\top} \boldsymbol{h}_j+\omega(t-t_j)+b\big)},
$
where the $j$-th hidden state in the RNN $\boldsymbol{h}_j$ is used to represent the history influence up to the nearest happened event $j$, and $w(t-t_j)$ represents the current influence. The $\boldsymbol{v}, \omega, b$ are trainable parameters \citep{Du2016}.
    
{\it Neural Hawkes Process} (\texttt{NHP}) specifies the conditional intensity function in point processes using a continuous-time LSTM, denoted as
$
\lambda^*(t) = f(\boldsymbol{\nu}^\top \boldsymbol{h}_t),
$
where the hidden state of the LSTM up to time $t$ represents the history influence, the $f(\cdot)$ is a softplus function which ensure the positive output given any input \citep{Mei2017}. 
    
{\it Self-Attentive Hawkes Process} (\texttt{SAHP}) adopts self-attention mechanism to model the historical information in the conditional intensity function, which is specified as
$
\lambda^*(t) = \text{softmax}\big(\mu+ \alpha \exp\{\omega(t-t_j)\}\big),
$
where $\mu, \alpha, w$ are computed via three non-linear mappings:
$\mu = \text{softplus}(\boldsymbol{h}W_\mu)$,
$\alpha = \text{tanh}(\boldsymbol{h}W_\alpha)$,
$\omega = \text{softplus}(\boldsymbol{h}W_\omega)$. 
The $W_\mu, W_\alpha, W_\omega$ are trainable parameters \citep{Zhang2019}. 

{\it Hawkes Process} (\texttt{HP}) specifies the conditional intensity function as
$
    \lambda^*(t) = \mu + \alpha \sum_{t_j < t} \beta \exp\{- \beta (t - t_j)\},
$
where parameters $\mu, \alpha, \beta$ can be estimated via maximizing likelihood \citep{Hawkes1971}.

\vspace{-0.1in}
\subsection{Synthetic data sets}
\label{sec:synthetic-data}
\vspace{-0.1in}

The synthetic data are obtained by the following four generative processes:
(1) \emph{Hawkes process}: the conditional intensity function is given by $\lambda^*(t) = \mu + \alpha \sum_{t_j < t} \beta \exp{- \beta((t-t_j))}$, where $\mu=10$, $\alpha=1$, and $\beta=1$;
(2) \emph{self-correction point process}: the conditional intensity function is given by $\lambda^*(t) = \exp{(\mu t - \sum_{t_i<t}\alpha)}$, where $\mu=10$, $\alpha=1$;
(3) \emph{non-homogeneous Poisson $1$}: The intensity function is given by $\lambda^*(t) = c \cdot\Phi(t-0.5)\cdot U[0,1]$ where $c=100$ is the sample size, the $\Phi(\cdot)$ is the PDF of standard normal distribution, and $U[a,b]$ is uniform distribution between $a$ and $b$;
(4) \emph{non-homogeneous Poisson $2$}: The intensity function is a composition of two normal functions, where $\lambda^*(t) = c_1 \cdot\Phi(6(t-0.35))\cdot U[0,1] + c_2 \cdot\Phi(6(t-0.75))\cdot U[0,1]$, where $c_1=50$, $c_2=50$.
Each synthetic data set contains 5,000 sequences with an average length of 30, where each data point in the sequence only contains the occurrence time of the event.

\vspace{-0.1in}
\subsection{Real data sets}
\label{append:real-data}
\vspace{-0.1in}

{\it Traffic Congestions} (traffic): We collect the data of traffic congestions from the Georgia Department of Transportation \citep{GDOT} over 178 days from 2017 to 2018, including 15,663 congestion events recorded by 86 observation sites. Each event consists of time, location, and congestion level. We partition the data into 178 sequences by day, and each sequence has an average length of 88. 

{\it Electrical Medical Records} (MIMIC-III): Medical Information Mart for Intensive Care III (MIMIC-III) \citep{Johnson2016} contains de-identified clinical visit time records from 2001 to 2012 for more than 40,000 patients. We select 2,246 patients with at least three visits. Each patient's visit history will be considered an event sequence, and each clinical visit will be considered an event.

{\it Financial Transactions} (stock): We collected data from the NYSE of the high-frequency transactions. It contains 0.7 million transaction records, each of which records the time (in milliseconds) and the possible action (sell or buy). We partition the raw data into 5,756 sequences with an average length of 48 by days.

{\it Memes} (meme): MemeTracker \citep{Leskovec2007}  tracks the meme diffusion over public media, which contains more than 172 million news articles or blog posts. The memes are sentences, such as ideas, proverbs, and recorded when it spreads to specific websites. We randomly sample 22,003 sequences of memes with an average length of 24.

\vspace{-0.1in}
\subsection{Synthetic data experiment results}
\label{sec:sim-data}
\vspace{-0.1in}

In the following experiments with synthetic data, we confirmed that our deep attention point process model could capture synthetic events' spatio-temporal dynamics. 
We first summarized the mean square error of recovering the true intensity in Table~\ref{tab:pred-error}, where our methods achieve the minimal error in recovering intensities. 
We also visualized recovered intensity over time given a randomly-picked sequence from each data set in Figure~\ref{fig:sim-res-intensity}. The solid grey line represents the true intensity of the sequence. The result shows that our methods can accurately recover the temporal intensity. It is noteworthy that our approach can also capture the sequences' dynamics with non-homogeneous temporal intensity, as shown in Figure~\ref{fig:sim-res-intensity} (a), (b), which is extremely challenging to characterize by the other baselines. 
Besides, we emphasize that the online version of our approach (ODAPP) also shows competitive performances against other methods, where only 50\% of events are used.

To further investigate the effects of deep Fourier kernel on discrete event modeling, we consider the following two ablation studies: we replace the deep Fourier kernel score function in the DAPP by
(a) \texttt{DAPP+dot-prod}: the conventional dot-product; 
(b) \texttt{DAPP+NN}: a fully-connected neural network (with the same configuration as the generator in deep Fourier kernel), where the network's input is the concatenation of projections of two events, and the scalar output is the score. 
As we can see from Table~\ref{tab:pred-error}, the non-linearity of the score function plays a pivotal role in modeling the spatio-temporal dynamics between these events and drastically reduces the mean square error. In particular, the deep Fourier kernel enjoys a greater expressive power in representing non-linear triggering effects comparing to a simple neural network. The above result confirms that the combination of the attention and our deep Fourier kernel-based score leads to event modeling and prediction success. 


\vspace{-0.1in}
\subsection{Real data experiment results}
\label{sec:real-data}
\vspace{-0.1in}

This section evaluates the performance of our methods on real-world data sets from a diverse range of domains, including a spatio-temporal data set and three other temporal data sets.
Due to a lack of true knowledge of intensity in real data, the recovery error is unavailable. 
Here, we reported the average log-likelihood of each method over training epochs on the testing data in Figure~\ref{fig:real} and summarized the highest average log-likelihood each method could obtain after the convergence in Table~\ref{tab:loglik}. 
As we can see, our DAPP and ODAPP outperform the other alternatives with higher average log-likelihood values on various data sets. 

\vspace{-0.1in}
\section{Conclusion}
\label{sec:conclusion}
\vspace{-0.1in}

We proposed an attention-based spatio-temporal processes model with a deep Fourier kernel, where the spectrum represented by neural networks captures complex non-linear dependence on past events. As demonstrated by our experiments with synthetic and real data, our method achieves competitive performance in achieving higher log-likelihood and smaller recovery error for conditional intensity function of a point process compared to the state-of-the-art.

\newpage

%

\bibliography{refs}

\appendix
\onecolumn
\title{Deep Fourier Kernel for Self-Attentive Point Processes: \\
Supplementary Materials}

\section{Proof for Proposition~\ref{prop:score-reformulation}}
\label{append:proof-prop-1}

For the notational simplicity, we omit all the index of attention head $(k)$ and denote $W_u x$ as $x$.
First, since both $\nu$ and $p$ are real-valued, it suffices to consider only the real portion of $e^{ix}$ when invoking Theorem~\ref{thm1}. Thus, using $\text{Re}[e^{ix}] = \text{Re}[\cos(x) + i \sin(x)] = cos(x)$, we have
\[
    \nu(x, x^\prime) = \text{Re}[\nu(x, x^\prime)] = \int_\Omega p_\omega(\omega) \cos(\omega^\top (x - x^\prime)) d\omega.
\]
Next, we have
\begin{align*}
    &~ \int_\Omega p_\omega(\omega) \cos\left(\omega^\top (x - x^\prime)\right) d\omega\\
    \overset{(i)}{=} &~ \int_\Omega p_\omega(\omega) \cos\left(\omega^\top (x - x^\prime)\right) d\omega + \int_\Omega \int_0^{2\pi} \frac{1}{2\pi} p_\omega(\omega) \cos\left( \omega^\top (x+x^\prime) + 2 b_u \right) db_u d\omega \\
    = &~ \int_\Omega \int_0^{2\pi} \frac{1}{2\pi} p_\omega(\omega) \left[ \cos\left(\omega^\top (x - x^\prime)\right) + \cos\left( \omega^\top (x+x^\prime) + 2 b_u \right) \right] db_u d\omega\\
    = &~ \int_\Omega \int_0^{2\pi} \frac{1}{2\pi} p_\omega(\omega) \left[ 2 \cos(\omega^\top x + b_u) \cdot \cos(\omega^\top x^\prime + b_u) \right] db_u d\omega\\
    = &~ \int_\Omega p_\omega(\omega) \int_0^{2\pi} \frac{1}{2\pi} \left[ \sqrt{2} \cos(\omega^\top x + b_u) \cdot \sqrt{2} \cos(\omega^\top x^\prime + b_u) \right] db_u d\omega\\
    = &~ \mathbb{E} \left [ \phi_\omega(x) \cdot \phi_\omega(x^\prime) \right ]. 
\end{align*}
where $\phi_\omega(x) \coloneqq \sqrt{2} \cos(\omega^\top x + b_u)$, $\omega$ is sampled from $p_\omega$, and $b_u$ is uniformly sampled from $[0, 2\pi]$. The equation $(i)$ holds since the second term equals to 0 as shown below:
\begin{align*}
    \int_\Omega \int_0^{2\pi} p_\omega(\omega) \cos\left( \omega^\top (x+x^\prime) + 2 b_u \right) db_u d\omega
    = &~ \int_\Omega p_\omega(\omega) \int_0^{2\pi} \cos\left( \omega^\top (x+x^\prime) + 2 b_u \right) db_u d\omega \\
    = &~ \int_\Omega p_\omega(\omega) \cdot 0 \cdot d\omega = 0.
\end{align*}
Therefore, we can obtain the result in Proposition~\ref{prop:score-reformulation}.

\section{Proof for Proposition~\ref{prop:score-convergence}}
\label{append:proof-prop-2}

Similar to the proof in Appendix~\ref{append:proof-prop-1}, we omit all the index of attention head $(k)$ and denote $W_u x$ as $x \in \mathcal{X}$ for the notational simplicity. Recall that we denote $R$ as the radius of the Euclidean ball containing $\mathcal{X}$ in Section~\ref{sec:score}. In the following, we first present two useful lemmas.

\begin{lemma}
\label{lemma:prop-2-lemma-1}
Assume $\mathcal{X} \subset \mathbb{R}^d$ is compact. Let $R$ denote the radius of the Euclidean ball containing $\mathcal{X}$, then for the kernel-induced feature mapping $\Phi$ defined in \eqref{eq:score-def-3}, the following holds for any $0 < r \le 2R$ and $\epsilon > 0$:
\[
    \mathbb{P}\left\{\underset{x, x^\prime \in \mathcal{X}}{\sup} \left | \Phi(x)^\top \Phi(x^\prime) - \nu(x, x^\prime) \right | \ge \epsilon \right\} \le 2 \mathcal{N}(2R, r) \exp\left\{ - \frac{D \epsilon^2}{8}\right\} + \frac{4r\sigma_p}{\epsilon}.
\]
where $\sigma_p^2 = \mathbb{E}_{\omega \sim p_\omega} [\omega^\top \omega] < \infty$ is the second moment of the Fourier features, and $\mathcal{N}(R, r)$ denotes the minimal number of balls of radius $r$ needed to cover a ball of radius $R$.
\end{lemma}

\begin{proof}[Proof of Lemma \ref{lemma:prop-2-lemma-1}]
Now, define $\Delta = \{\delta: \delta = x - x^\prime,~, x,x^\prime \in \mathcal{X} \}$ and note that $\Delta$ is contained in a ball of radius at most $2R$. $\Delta$ is a closed set since $\mathcal{X}$ is closed and thus $\Delta$ is a compact set. 
Define $B = \mathcal{N}(2R, r)$ the number of balls of radius $r$ needed to cover $\Delta$ and let $\delta_j$, for $j \in [B]$ denote the center of the covering balls. 
Thus, for any $\delta \in \Delta$ there exists a $j$ such that $\delta = \delta_j + r^\prime$ where $|r^\prime| < r$.

Next, we define $S(\delta) = \Phi(x)^\top \Phi(x^\top) - \nu(x, x^\prime)$, where $\delta = x - x^\prime$. Since $S$ is continuously differentiable over the compact set $\Delta$, it is $L$-Lipschitz with $L = \sup_{\delta \in \Delta} || \nabla S(\delta) ||$. Note that if we assume $L < \frac{\epsilon}{2r}$ and for all $j \in [B]$ we have $|S(\delta_j)| < \frac{\epsilon}{2}$, then the following inequality holds for all $\delta = \delta_j + r^\prime \in \Delta$:
\begin{equation}
    |S(\delta)| = |S(\delta_j + r^\prime)| \le L |\delta_j - (\delta_j + r^\prime)| + |S(\delta_j)| \le rL + \frac{\epsilon}{2} < \epsilon.
    \label{eq:prop-2-part-1}
\end{equation}
The remainder of this proof bounds the probability of the events $L > \epsilon / (2r)$ and $|S(\delta_j)|\ge \epsilon/2$. Note that all following probabilities and expectations are with respect to the random variables $\omega_1, \dots, \omega_D$.

To bound the probability of the first event, we use Proposition~\ref{prop:score-reformulation} and the linearity of expectation, which implies the key fact $\mathbb{E}[\nabla(\Phi(x)^\top \Phi(x^\prime))] = \nabla \nu(x, x^\top)$. We proceed with the following series of inequalities:
\begin{align*}
    \mathbb{E}\left[L^2\right] 
    & = \mathbb{E}\left[\underset{\delta\in\Delta}{\sup}||\nabla S(\delta)||^2\right]\\
    & = \mathbb{E}\left[\underset{x, x^\prime \in \mathcal{X}}{\sup}||\nabla(\Phi(x)^\top \Phi(x^\prime)) - \nabla \nu(x, x^\prime) ||^2\right]\\
    & \overset{(i)}{\le} 2 \mathbb{E}\left[\underset{x, x^\prime \in \mathcal{X}}{\sup}||\nabla(\Phi(x)^\top \Phi(x^\prime))||^2\right] + 
    2 \underset{x, x^\prime \in \mathcal{X}}{\sup}||\nabla \nu(x, x^\prime) ||^2\\
    & = 2 \mathbb{E}\left[\underset{x, x^\prime \in \mathcal{X}}{\sup}||\nabla(\Phi(x)^\top \Phi(x^\prime))||^2\right] + 
    2 \underset{x, x^\prime \in \mathcal{X}}{\sup}||\mathbb{E}\left[ \nabla(\Phi(x)^\top \Phi(x^\prime))\right] ||^2\\
    & \overset{(ii)}{\le} 4 \mathbb{E}\left[\underset{x, x^\prime \in \mathcal{X}}{\sup}||\nabla(\Phi(x)^\top \Phi(x^\prime))||^2\right],
\end{align*}
where the first inequality $(i)$ holds due to the the inequality $||a + b||^2 \le 2 ||a||^2 + 2 ||b||^2$ (which follows from Jensen’s inequality) and the subadditivity of the supremum function. The second inequality $(ii)$ also holds by Jensen’s inequality (applied twice) and again the subadditivity of supremum function. Furthermore, using a sum-difference trigonometric identity and computing the gradient with respect to $\delta = x - x^\prime$, yield the following for any $x, x^\prime \in \mathcal{X}$:
\begin{align*}
    \nabla(\Phi(x)^\top \Phi(x^\prime)) 
    & = \nabla \left(\frac{1}{D} \sum_{i=1}^D \cos(\omega_i^\top (x-x^\prime))\right) \\
    & = \frac{1}{D} \sum_{i=1}^D \omega_i \sin(\omega_i^\top (x-x^\prime)).
\end{align*}
Combining the two previous results gives
\begin{align*}
    \mathbb{E}[L^2] 
    & \le 4 \mathbb{E}\left[ \underset{x, x^\prime \in \mathcal{X}}{\sup} ||\frac{1}{D} \sum_{i=1}^D \omega_i \sin(\omega_i^\top (x-x^\prime)) ||^2 \right]\\
    & \le 4 \underset{\omega_1, \dots, \omega_D}{\mathbb{E}} \left[ \left( \frac{1}{D}\sum_{i=1}^{D} ||\omega_i|| \right)^2 \right]\\
    & \le 4 \underset{\omega_1, \dots, \omega_D}{\mathbb{E}} \left[ \frac{1}{D}\sum_{i=1}^{D} ||\omega_i||^2 \right] = 4 \underset{\omega}{\mathbb{E}}[||\omega||^2] = 4 \sigma_p^2,
\end{align*}
which follows from the triangle inequality, $|\sin(\cdot)|\le 1$, Jensen’s inequality and the fact that the $\omega_j$s are drawn i.i.d. derive the final expression. Thus, we can bound the probability of the first event via Markov’s inequality:
\begin{equation}
    \mathbb{P}\left[ L \ge \frac{\epsilon}{2r} \right] \le \left( \frac{4r\sigma_p}{\epsilon} \right)^2.
    \label{eq:prop-2-part-2}
\end{equation}
To bound the probability of the second event, note that, by definition, $S(\delta)$ is a sum of $D$ i.i.d. variables, each bounded in absolute value by $\frac{2}{D}$ (since, for all $x$ and $x^\prime$, we have $|\nu(x, x^\prime)| \le 1$ and $|\Phi(x)^\top \Phi(x^\prime)| \le 1$), and $\mathbb{E}[S(\delta)] = 0$. Thus, by Hoeffding’s inequality and the union bound, we can write
\begin{equation}
    \mathbb{P}\left[\exists j \in [B]: |S(\delta_j)|\ge \frac{\epsilon}{2}\right] \le \sum_{j=1}^{B} \mathbb{P}\left[|S(\delta_j)|\ge \frac{\epsilon}{2}\right] \le 2 B \exp\left( -\frac{D\epsilon^2}{8} \right).
    \label{eq:prop-2-part-3}
\end{equation}
Combining \eqref{eq:prop-2-part-1}, \eqref{eq:prop-2-part-2}, \eqref{eq:prop-2-part-3}, and the definition of $B$ we have
\[
    \mathbb{P}\left[\underset{\delta \in \Delta}{\sup}|S(\delta_j)|\ge \epsilon\right] \le 2 \mathcal{N}(2 R, r) \exp\left\{ - \frac{D \epsilon^2}{8} \right\} + \left(\frac{4r\sigma_p}{\epsilon}\right)^2.
\]
\end{proof}

As we can see now, a key factor in the bound of the proposition is the covering number $N(2R,r)$, which strongly depends on the dimension of the space $N$. In the following proof, we make this dependency explicit for one especially simple case, although similar arguments hold for more general scenarios as well.

\begin{lemma}
\label{lemma:prop-2-lemma-2}
Let $\mathcal{X} \subset \mathbb{R}^d$ be a compact and let $R$ denote the radius of the smallest enclosing ball. Then, the following inequality holds:
\[
    \mathcal{N}(R, r) \le \left( \frac{3R}{r} \right)^d.
\]
\end{lemma}
\begin{proof}[Proof of Lemma \ref{lemma:prop-2-lemma-2}]
By using the volume of balls in $\mathbb R^d$, we already see that $R^d / (r/3)^d = (3R/r)^d$ is a trivial upper bound on the number of balls of radius $r/3$ that can be packed into a ball of radius $R$ without intersecting. Now, consider a maximal packing of at most $(3R/r)^d$ balls of radius $r/3$ into the ball of radius $R$. Every point in the ball of radius $R$ is at distance at most $r$ from the center of at least one of the packing balls. If this were not true, we would be able to fit another ball into the packing, thereby contradicting the assumption that it is a maximal packing. Thus, if we grow the radius of the at most $(3R/r)^d$ balls to $r$, they will then provide a (not necessarily minimal) cover of the ball of radius $R$.
\end{proof}

Finally, by combining the two previous lemmas, we can present an explicit finite sample approximation bound. We use lemma~\ref{lemma:prop-2-lemma-1} in conjunction with lemma~\ref{lemma:prop-2-lemma-2} with the following choice of $r$:
\[
    r = \left[ \frac{2(6R)^d \exp(- \frac{D\epsilon^2}{8})}{\left( \frac{4\sigma_p}{\epsilon} \right)^2} \right]^{\frac{2}{d+2}},
\]
which results in the following expression
\[
    \mathbb{P}\left[ \underset{\delta \in \Delta}{\sup} |S(\delta)| \ge \epsilon \right] \le 4 \left( \frac{24R\sigma_p}{\epsilon} \right)^{\frac{2d}{d+2}} \exp\left( -\frac{D\epsilon^2}{4(d+2)} \right).
\]
Since $32R \sigma_p/\epsilon \ge 1$, the exponent $2d / (d+2)$ can be replaced by 2, which completes the proof.

\section{Algorithm}
\label{append:learning}


\begin{algorithm}[!h]
\SetAlgoLined
    {\bfseries Input:} The data set $X = \{\boldsymbol{x}_j\}_{j=1,\dots,n}$ with $n$ samples, where each sample $\boldsymbol{x} = \{x_i\}_{i=1}^{N_T}$ is a series of events, $N_T$ is the number of events in the time horizon $T$\;
    Define the number of iterations $\eta$, the number of samples in a mini-batch $M$, and the number of random Fourier features $D$\;
    Initialize model parameters $\boldsymbol{\theta}_0 = \{W, b, \{\theta^{(k)}, W^{(k)}_u, W^{(k)}_v\}_{k=1,\dots,K}\}$; $l = 0$\; 
    \While{$l < \eta$}{
        Randomly draw $M$ sequences from $X$ denoted as $\widehat{X}_l = \{\boldsymbol{x}_j: \boldsymbol{x}_j \in \mathcal{X}\}_{j=1,\dots,M}$\;
        Generate $D$ Fourier features from $p_\omega$ denoted as $\widehat{\Omega}_l = \{\omega_k \coloneqq G(z;\theta), z \sim p_z \}_{k=1,\dots,D}$\;
        $\boldsymbol{\theta}_l \leftarrow$ Update $\boldsymbol{\theta}_l$ by maximizing \eqref{eq:log-likelihood} using stochastic gradient descent given $\widehat{X}_l, \widehat{\Omega}_l$\;
        $l \leftarrow l + 1$\;
    }
\caption{Learning for DAPP}
\label{alg:learning}
\end{algorithm}


\begin{algorithm}[!h]
\SetAlgoLined
  {\bfseries input} $\boldsymbol{\theta}, T, \mathcal{M}$\;
  {\bfseries output} A set of events $\mathcal{H}_t$ ordered by time.\;
  Initialize $\mathcal{H}_t = \emptyset$, $t=0$, $m \sim \texttt{uniform}(\mathcal{M})$\;
  \While{$t < T$}{
    Sample $u \sim \texttt{uniform}(0, 1)$; $m \sim \texttt{uniform}(\mathcal{M})$; $D \sim \texttt{uniform}(0, 1)$\;
    $x^\prime \leftarrow (t, m^\prime)$; $\bar{\lambda} \leftarrow \lambda(x^\prime | \boldsymbol{h}(x^\prime))$ given history $\mathcal{H}_t$\;
    $t \leftarrow t  - \ln u/\bar{\lambda}$\;
    $x \leftarrow (t, m)$; $\widetilde{\lambda} \leftarrow \lambda(x | \boldsymbol{h}(x))$ given history $\mathcal{H}_t$\;
    \If{$D \bar{\lambda} > \widetilde\lambda$}{
      $\mathcal{H}_t \leftarrow \mathcal{H}_t \cup \{(t, m)\}$; $m^\prime \leftarrow m$\;
    }
  }
\caption{Efficient thinning algorithm for DAPP}
\label{alg:thinning}
\end{algorithm}


\begin{algorithm}[!h]
\SetAlgoLined
    {\bfseries Input:} data $\boldsymbol{x} = \{x_i\}_{i=1}^\infty$, threshold $\eta$\;
    Initialize $\mathscr{A}_{0}^{(k)} = \emptyset, k=1,\dots, K$\;
    \For{$i=1$ {\bfseries to} $+\infty$.}{
        \For{$k=1$ {\bfseries to} $K$.}{
            $\mathscr{A}_{i}^{(k)} \leftarrow \mathscr{A}_{i-1}^{(k)} \cup \{x_i\}$\;
            Initialize $\mathscr{S}_{i}^{(k)} = \emptyset$, $\bar\nu_{j}^{(k)} = 0$\;
            \For{$j=1$ {\bfseries to} $i-1$}{
                $\mathscr{S}_{j}^{(k)} \leftarrow \mathscr{S}_{j}^{(k)} \cup \widetilde\nu^{(k)}(x_i, x_j)$\;
                $\bar\nu_{j}^{(k)} \leftarrow (\sum_{s \in \mathscr{S}_{j}^{(k)}} s) / |\mathscr{S}_{j}^{(k)}|$\;
            }
            \If{$i > \eta$}{
                $\mathscr{A}_{i}^{(k)} \leftarrow \mathscr{A}_{i-1}^{(k)} \setminus \underset{x_j: t_j < t_i}{\arg \min} \left\{ \bar\nu_{j}^{(k)} \right\}$\;
            }
        }
    }
\caption{Event selection for online attention}
\label{alg:online-attention}
\end{algorithm}


\end{document}